\documentclass{article}

\usepackage{microtype}
\usepackage{graphicx}
\usepackage{subfigure}
\usepackage{graphicx}
\usepackage{bm}
\usepackage{amsmath, amssymb}
\usepackage{amsthm}
\usepackage{breqn}
\usepackage[shortlabels]{enumitem}
\newtheorem{theorem}{Theorem}
\newtheorem{corollary}{Corollary}[theorem]
\theoremstyle{definition}
\newtheorem{definition}{Definition}
\usepackage{booktabs} 
\usepackage{siunitx}
\usepackage{listings}
\usepackage{subfigure}
\usepackage{comment}
\usepackage{url}

\usepackage{hyperref}

\usepackage[accepted]{icml2021}

\icmltitlerunning{Supervised Tree-Wasserstein Distance}

\begin{document}
\twocolumn[
\icmltitle{Supervised Tree-Wasserstein Distance}



\icmlsetsymbol{equal}{*}

\begin{icmlauthorlist}
\icmlauthor{Yuki Takezawa}{univ,riken}
\icmlauthor{Ryoma Sato}{univ,riken}
\icmlauthor{Makoto Yamada}{univ,riken}
\end{icmlauthorlist}

\icmlaffiliation{univ}{Kyoto University}
\icmlaffiliation{riken}{RIKEN AIP}

\icmlcorrespondingauthor{Yuki Takezawa}{yuki-takezawa@ml.ist.i.kyoto-u.ac.jp}

\icmlkeywords{Machine Learning, ICML}

\vskip 0.3in
]



\printAffiliationsAndNotice{}  

\begin{abstract}
To measure the similarity of documents, the
Wasserstein distance is a powerful tool, but it requires a high computational cost.
Recently, for fast computation of the Wasserstein distance, 
methods for approximating the Wasserstein distance using a tree metric have been proposed. These tree-based methods
allow fast comparisons of a large number of documents; however, they
are unsupervised and do not learn task-specific distances.
In this work, we propose the Supervised Tree-Wasserstein (STW) distance, 
a fast, supervised metric learning method based on the tree metric.
Specifically, we rewrite the Wasserstein distance on the tree metric by the parent--child relationships of a tree,
and formulate it as a continuous optimization problem using a contrastive loss.
Experimentally, we show that the STW distance can be computed fast,
and improves the accuracy of document classification tasks.
Furthermore, the STW distance is formulated by matrix multiplications, runs on a GPU, and is suitable for batch processing.
Therefore, we show that the STW distance is extremely efficient 
when comparing a large number of documents.
\end{abstract}

\section{Introduction}
The Wasserstein distance is a powerful tool for measuring distances between distributions. 
It has recently been applied in many fields, such as feature matching \cite{sarlin2020superGlue, yabin2020semantic}, generative models \cite{kolouri2019sliced}, similarity metrics \cite{kusner2015from, gao2016supervised, yurochkin2019hierarchical}, and so on.

The Wasserstein distance can be computed by solving the optimal transport problem.
For similarity metrics of documents,
\citet{kusner2015from} proposed the Word Mover's Distance (WMD).
Given the word embedding vectors \cite{mikolov2013distributed} and a normalized bag-of-words, the 
WMD is the cost of the optimal transport 
between two documents in the word embedding space.
WMD has been used for document classification tasks
and has achieved high $k$-nearest neighbors ($k$NN) accuracy.

To solve the optimal transport problem, linear programming can be used.
However, using linear programming requires cubic time with respect to the number of data points \cite{pele2009fast}. 
\citet{cuturi2013sinkhorn} proposed to add entropic regularization to the optimal transport problem, 
which can be solved by using a matrix scaling algorithm in quadratic time.
To further reduce the computational cost of the optimal transport problem,
there are two main strategies.
(1) The first approach is to relax the constraint of the optimal transport problem. Specifically, 
\citet{kusner2015from} relaxed the constraints of the optimal transport problem and transport the mass of each coordinate to the nearest coordinate, called the Relaxed WMD (RWMD). 
\citet{kubilay2019linear} attached additional constraints to RWMD and proposed a more accurate approximation of WMD.
(2) The second approach is to construct a tree metric and compute the Wasserstein distance on the tree metric (tree-Wasserstein distance).
\citet{indyk2003fast} proposed a method to embed the coordinates into the tree metric, called Quadtree.
Recently, \citet{le2019tree} proposed a method to sample tree metrics and achieved a high accuracy in document classification tasks.
\citet{arturs2020scalable} proposed a more accurate method than Quadtree.
These tree-based methods aim to approximate the Wasserstein distance on the Euclidean metric with the tree-Wasserstein distance.
The tree-Wasserstein distance can be computed in linear time with respect to the number of nodes in the tree 
and can quickly compare a large number of documents. 

In general, the similarity between documents must be designed in a task-specific manner.
However, the methods mentioned above are unsupervised and do not learn task-specific distances.
\citet{gao2016supervised} proposed supervised metric learning based on WMD, called Supervised WMD (S-WMD).
S-WMD learns a task-specific distance by leveraging the label information of documents, and improves the $k$NN accuracy. 
However, it requires quadratic time to compute S-WMD
and there is no supervised metric learning for the tree-Wasserstein distance.
Moreover, for the tree-Wasserstein distance, it is challenging to construct the tree metric by leveraging the label information of documents.

In this work, 
we propose the \textit{Supervised Tree-Wasserstein} (STW) distance,
a fast supervised metric learning method for the tree metric.
To this end, we propose the \textit{soft tree-Wasserstein distance}, which is a soft variant of the tree-Wasserstein distance.
Specifically, we rewrite the tree-Wasserstein distance by the probability of the parent–-child relationships of a tree.
We then consider learning the probability of the parent--child relationships of a tree by leveraging the label information of documents.
By virtue of the soft tree-Wasserstein distance, the STW distance is end-to-end trainable using backpropagation 
and is formulated only by matrix multiplications, which can be implemented with simple operations on a GPU. 
Thus, the STW distance is suitable for batch processing and can simultaneously compare multiple documents.
Through synthetic and real-world experiments on document classification tasks, 
we show that the STW distance
can build a tree that represents the task-specific distance and has improved accuracy.
Furthermore, we show that the STW distance is more efficient than the existing methods for computing Wasserstein distances,
especially when comparing a large number of documents.

Our contributions are as follows:
\begin{itemize}
    \item We propose a soft variant of the tree-Wasserstein distance, which is differentiable with respect to the probability of the parent--child relationships of a tree. It can be computed by simple operations on a GPU and is suitable for batch processing.
    \item Using the soft variant of the tree-Wasserstein distance, we propose fast supervised metric learning for a tree metric, which is formulated as a continuous optimization problem.  
    \item Experimentally, we show that our method is fast and improves the accuracy of document classification tasks.
\end{itemize}

\textbf{Notation:}
In the following sections,
we write $\mathbf{1}_n$ for an $n$-dimensional vector with all ones, 
$\mathbf{0}_n$ for an $n$-dimensional vector with all zeros,
$\mathbf{I}$ for the identity matrix,
and $\delta$ for the Dirac delta function.

\section{Related Work}
\label{sec:related_work}
In this section, we introduce the existing Wasserstein distances and the methods for continuous optimization for learning a tree structure,
and then present their drawbacks.
 
\subsection{Wasserstein Distances}
Given a simplex $\mathbf{a} \in \mathbb{R}_{+}^{n}$ and $\mathbf{b} \in \mathbb{R}_{+}^{m}$, 
we write $U(\mathbf{a}, \mathbf{b})$ for the transport polytope of $\mathbf{a}$ and $\mathbf{b}$ as follows: 
\begin{equation*}
    U(\mathbf{a}, \mathbf{b}) = \{ \mathbf{T} \in \mathbb{R}_{+}^{n \times m} \; | \; \mathbf{T} \mathbf{1}_m = \mathbf{a}, \mathbf{T}^\top \mathbf{1}_n = \mathbf{b} \},
\end{equation*}
Given a cost $c(\mathbf{x}_i, \mathbf{x}_j)$ between coordinates $\mathbf{x}_i$ and $\mathbf{x}_j$,
the optimal transport problem between $\mathbf{a}$ and $\mathbf{b}$ is defined as follows:
\begin{equation*}
    \min_{\mathbf{T} \in U(\mathbf{a}, \mathbf{b})}  \sum_{i,j} \mathbf{T}_{i,j} \; c(\mathbf{x}_i, \mathbf{x}_j).
\end{equation*}
If $c(\mathbf{x}_i, \mathbf{x}_j)$ is a metric, then
the cost of the optimal transport is a metric, which is a special case of Wasserstein distances.

In document classification tasks, 
given word embedding vectors $\mathbf{x}_i$ and $\mathbf{x}_j$,
\citet{kusner2015from} defined the cost $c(\mathbf{x}_i, \mathbf{x}_j) = \| \mathbf{x}_i - \mathbf{x}_j \|^2_2$
and simplex $\mathbf{a}$ and $\mathbf{b}$ as the normalized bag-of-words,
and proposed to use the optimal transport cost as the dissimilarity of documents, called Word Mover's Distance (WMD).
To further improve the classification accuracy,
\citet{gao2016supervised} proposed supervised metric learning based on WMD, called Supervised WMD (S-WMD).
S-WMD transforms word embedding vectors and re-weights the bag-of-words via supervised learning.

To solve the optimal transport problem, linear programming can be used.
However, using linear programming requires cubic time with respect to the number of coordinates \cite{pele2009fast}. 
To reduce this time complexity, 
\citet{cuturi2013sinkhorn} proposed the entropic regularized optimal transport, which is called the Sinkhorn algorithm
and can be solved in quadratic time.

\textbf{Tree-Wasserstein Distances:}
Given a tree $\mathcal{T} = (\bm{V}, \bm{E})$ rooted at $v_1$ with non-negative edge lengths, the tree metric $d_{\mathcal{T}}$ between two nodes is the total length of the path between the nodes.
Let $\Gamma(v)$ be a set of nodes contained in the subtree of $\mathcal{T}$ rooted at $v \in \bm{V}$.
For all $v \in \bm{V} \setminus \{ v_1 \}$, there exists a unique node $u \in \bm{V}$ which is the parent node of $v$
and we write $w_v$ for the length of the edge from $v$ to its parent node.
Given two measures $\mu$ and $\nu$ supported on $\mathcal{T}$,
the tree-Wasserstein distance between $\mu$ and $\nu$ is calculated as follows:
\begin{equation}
    W_{d_{\mathcal{T}}}(\mu, \nu) = \sum_{v \in \bm{V} \setminus \{ v_1\}} w_v \left| \mu(\Gamma(v)) - \nu(\Gamma(v)) \right|.
\label{eqn:tree_wasserstein}
\end{equation}
The parent node of the root $v_1$ does not exist, and the length of the edge $w_{v_1}$ is not defined.
However, because $\mu(\Gamma(v_1)) = \nu(\Gamma(v_1)) = 1$, we define $w_{v_1}=1$ for simplicity; 
the tree-Wasserstein distance can be written as $W_{d_{\mathcal{T}}}(\mu, \nu) = \sum_{v \in \bm{V}} w_v \left| \mu(\Gamma(v)) - \nu(\Gamma(v)) \right|$.
The key property of the tree-Wasserstein distance is that it can be computed in linear time with respect to the number of nodes.
Furthermore, the tree-Wasserstein distance between $\mu$ and $\nu$ is regarded as the L1 distance between their corresponding $|\bm{V}|$-dimensional vectors whose elements corresponding to $v$ are $w_v \mu(\Gamma(v))$ and $w_v \nu(\Gamma(v))$.
In practice, these embedding vectors are sparse. This allows for faster implementation \cite{arturs2020scalable}.
In the unbalanced setting,
\citet{sato2019unbalanced} proposed a method to compute the tree-Wasserstein distance in quasi-linear time.

To compute the tree-Wasserstein distance, we need to construct a tree metric.
\citet{indyk2003fast} proposed a method to embed the coordinates into the tree metric in the context of image retrieval, which is called Quadtree.
\citet{le2019tree} proposed the tree-sliced Wasserstein (TSW) distance, which is a variant of the sliced-Wasserstein distance \cite{rabin2011wasserstein, kolouri2018sliced, kolouri2019generalized, deshpande2019max}.
The TSW distance is the average of the tree-Wasserstein distances on the sampled tree metrics.
Recently, \citet{arturs2020scalable} proposed Flowtree,
which computes the optimal flow on Quadtree,
then computes the cost of the optimal flow on the ground metric, unlike Quadtree and the TSW distance.
Flowtree is slower than Quadtree in computing the optimal flow, 
but can theoretically approximate the Wasserstein distance more accurately.
These previous works aimed to approximate the Wasserstein distance on the Euclidean metric with the tree-Wasserstein distance.
In contrast to these previous works, 
our goal is not to approximate the ground metric, but to construct a tree metric that represents the task-specific distance 
by leveraging the label information of the documents; so that the tree-Wasserstein distance between documents with the same label is small, and the tree-Wasserstein distance between documents with different labels is large.

\subsection{Continuous Optimization for a Tree}
When solving the task of learning a tree structure as a continuous optimization problem, 
learning in hyperbolic space is highly related.
Hyperbolic space has a property that is similar to that of a tree, where the volume increases exponentially with the radius, 
and the number of nodes increases exponentially with the depth of the tree.
Using this property, 
various methods that solve continuous optimization to learn a tree structure by representing the nodes with coordinates in hyperbolic space have been proposed \cite{nickel2017poincare, ganea2018hyperbolic}.
In hierarchical clustering, 
\citet{monath2019gradient, chami2020from} 
formulated the probability or the coordinates of the lowest common ancestors in hyperbolic space 
and constructed a tree by minimizing a soft variant of Dasgupta's cost \cite{dasgupta2016cost}, 
which is the well-known cost for hierarchical clustering.
However, these methods are not applicable to the tree-Wasserstein distance 
because it is necessary to formulate whether a node is contained in a subtree (i.e., $\Gamma(v)$). 
In contrast to these works,
we introduce the conditions of an adjacency matrix to be the adjacency matrix of a tree,
formulate the probability that a node is contained in a subtree, and
then propose a continuous optimization problem with respect to the adjacency matrix.

\section{Proposed Method}
\label{sec:proposed_method}
In this section, 
we first introduce a soft variant of the tree-Wasserstein distance;
then we propose the STW distance.

\subsection{Problem Setting}
We have a finite size vocabulary set  $\bm{Z}=\{ z_1, z_2, \ldots, z_{N_{\text{leaf}}}\}$ consisting of $N_{\text{leaf}}$ words
and a training dataset $\mathcal{D} = \{(\mathbf{a}_i, y_i)\}_{i=1}^{M}$ 
where $N_{\text{leaf}}$-dimensional vector $\mathbf{a}_i = (a_i^{(1)}, a_i^{(2)}, \ldots, a_i^{(N_{\text{leaf}})})^\top \in [0,1]^{N_{\text{leaf}}}$ is the normalized bag-of-words (i.e., $\mathbf{a}_i^\top \mathbf{1}_{N_{\text{leaf}}} =1$), and $y_i \in \mathbb{N}$ is a label of document $i$.
In the following sections, we assign words to leaf nodes of the tree, as in Quadtree and the TSW distance.
We refer to the nodes corresponding to each word as \textit{\text{leaf}} nodes
and the nodes not corresponding to any word as \textit{internal} nodes. 
Note that leaf nodes have no child nodes, but there may be internal nodes that do not have child nodes.
To construct the tree metric by leveraging the label information of documents, 
assume that we have a set of nodes $\bm{V} = \{v_1, v_2, \ldots , v_N\}$, in which $v_1$ is the root.
We consider constructing the tree metric by learning the parent–child relationships of these nodes.
Let $N_{\text{in}}$ be the number of internal nodes ($N = N_{\text{in}} + N_{\text{leaf}}$). 
$\bm{V}_{\text{in}} = \{v_1, v_2, \ldots, v_{N_{\text{in}}} \}$ is a set of internal nodes. 
$\bm{V}_{\text{leaf}} = \{v_{N_{\text{in}} + 1}, \ldots, v_{N} \}$ is a set of leaf nodes. 
$w_v$ is the length of an edge from $v$ to the parent node of $v$.
For simplicity, we define $w_{v_1}=1$.
We assume that the word $z_i$ corresponds to $v_{N_{\text{in}} + i}$.
We denote the training dataset using the discrete measure $\mathcal{D} = \{ (\mu_i, y_i)\}_{i=1}^M$,
where $\mu_i = \sum_j a_i^{(j)} \delta({v_{N_{\text{in}} + j}}, \cdot)$ is the discrete measure that represents the document $i$.

\subsection{Soft Tree-Wasserstein Distance}
Our goal is to construct a tree metric such that 
the tree-Wasserstein distance between documents with the same label is small and 
the distance between documents with different labels is large.
To achieve this, we first show the conditions of the parent--child relationships of a tree,
formulate the probability that a node is contained in a subtree using these conditions,
and then propose a soft variant of the tree-Wasserstein distance.

The parent--child relationships of a tree with a specific root can be represented by the adjacency matrix of the directed tree, which has edges from child nodes to their parent nodes.
We show the conditions for an adjacency matrix to be an adjacency matrix of a tree.
\begin{theorem}
If the adjacency matrix $\mathbf{D}_{\text{par}} \in \{0, 1\}^{N \times N}$ of a directed graph $G=(\bm{V}, \bm{E})$ satisfies the following conditions:
\begin{enumerate}
\renewcommand{\labelenumi}{(\arabic{enumi})}
\setlength{\itemsep}{0.1cm}
    \item $\mathbf{D}_{\text{par}}$ is a strictly upper triangular matrix. \label{cond:strictly_upper}
    \item $\mathbf{D}_{\text{par}}^\top \mathbf{1}_N = (0, 1, \cdots ,1)^\top$. \label{cond:eqn}
\end{enumerate}
then $G$ is a directed tree with $v_1$ as the root.
\label{th:condition_of_adjacency_matrix}
\end{theorem}
Appendix details the proof.
To introduce a soft variant of the tree-Wasserstein distance, we relax
$\mathbf{D}_{\text{par}} \in \{0, 1\}^{N \times N}$ to $\mathbf{D}_{\text{par}} \in \left[0, 1\right]^{N \times N}$ while satisfying the conditions of Theorem \ref{th:condition_of_adjacency_matrix}.
In $\mathbf{D}_{\text{par}}$, 
the elements in the first column are all zero; 
in the second and subsequent columns, 
the sum of the elements in each column is one.
In other words, the element in the $i$-th row and $j$-th column of $\mathbf{D}_{\text{par}}$ is the probability that $v_i$ is a parent of $v_j$.
The elements in the $i$-th row and $j$-th column of $\mathbf{D}^{k}_{\text{par}}$ denotes the probability 
that there exists a path from $v_j$ to $v_i$ with $k$ steps.
The element in the $i$-th row and $j$-th column of the sum of the infinite geometric series is 
the probability that there exists a path from $v_j$ to $v_i$.
In other words,
it means the probability that $v_j$ is contained in the subtree rooted at $v_i$.
We refer to this probability as $P_{\text{sub}}(v_j | v_i)$ and define it as follows:
\begin{align}
\begin{split}
   P_{\text{sub}}(v_j | v_i) &= \left[ \sum_{k=0}^\infty \mathbf{D}_{\text{par}}^k \right]_{i, j} = \left[ (\mathbf{I} - \mathbf{D}_{\text{par}})^{-1} \right]_{i,j}.
\label{eqn:psub}
\end{split}
\end{align}
$\mathbf{D}_{\text{par}}$ is a nilpotent matrix
because it is an upper triangular matrix and all the diagonal elements are zero.
Therefore, the sum of the infinite geometric series converges to $(\mathbf{I} - \mathbf{D}_{\text{par}})^{-1}$.
We show more details in the Appendix.
By using this probability,
we define the \textit{soft tree-Wasserstein distance} $W_{d_{\mathcal{T}}}^{\text{soft}} (\mu_i, \mu_j)$ as follows:
\begin{multline}
\label{eqn:soft_tree_wasserstein}
    W_{d_{\mathcal{T}}}^{\text{soft}} (\mu_i, \mu_j) \\
    = \sum_{v \in \bm{V}} w_v \left| \sum_{x \in \bm{V}_{\text{leaf}}}\!\!\!P_{\text{sub}}(x | v) \left( \mu_i (x) - \mu_j(x) \right) \right|_{\alpha}, 
\end{multline}
where $| \cdot |_{\alpha}$ is a smooth approximation of the L1 norm, defined as follows:
\begin{equation*}
    | x |_\alpha = \frac{x (e^{\alpha x} - e^{- \alpha x})}{2 + e^{\alpha x} + e^{- \alpha x}}.
\end{equation*}
It has been shown that if $\alpha$ approaches $\infty$, then $|\cdot|_{\alpha}$ converges to the L1 norm \cite{smoothabs}. 
Other differentiable approximations for the L1 norm can also be used.
The soft tree-Wasserstein distance satisfies the identity of indiscernibles and the symmetry,
but does not satisfy the triangle inequality, because $| \cdot |_{\alpha}$ does not satisfy the triangle inequality.
Thus, the soft tree-Wasserstein distance is not a metric. 
However, the soft tree-Wasserstein distance satisfies the following theorem;
the proof is shown in the Appendix.
\begin{theorem}
If the tree metric is given and $\alpha$ approaches $\infty$, then the soft tree-Wasserstein distance converges to the tree-Wasserstein distance.
\label{th:soft}
\end{theorem}

\subsection{Fast Computation Method}
Because the size of $\mathbf{D}_{\text{par}}$ is large, 
calculating the inverse matrix in Eq. \eqref{eqn:psub} has high computational cost and memory consumption.
Next, we introduce a method to reduce this cost by utilizing the property of $\mathbf{D}_{\text{par}}$.

We arranged the index of nodes such that the index of an internal node was less than the index of a leaf node.
As pointed out earlier, leaf nodes have no child nodes.
Then, the lower block of $\mathbf{D}_{\text{par}}$ is a zero matrix 
and $\mathbf{D}_{\text{par}}$ can be partitioned into four blocks as follows:
\begin{equation}
    \mathbf{D}_{\text{par}} = \begin{pmatrix}
    \mathbf{D}_1 & \mathbf{D}_2 \\
    \bm{0} & \bm{0}
    \end{pmatrix},
\label{eqn:block}
\end{equation}
where $\mathbf{D}_1$ is an $N_{\text{in}} \times N_{\text{in}}$ matrix, and $\mathbf{D}_2$ is an $N_{\text{in}} \times N_{\text{leaf}}$ matrix.
$\mathbf{D}_1$ denotes the parent--child relationships of a tree consisting of internal nodes, 
and $\mathbf{D}_2$ represents which internal nodes the leaf nodes connect to.
Utilizing this property and the constraints of $\mathbf{D}_{\text{par}}$,
we can calculate the inverse matrix as follows:
\begin{equation}
    (\mathbf{I} - \mathbf{D}_{\text{par}})^{-1} = \begin{pmatrix}
    (\mathbf{I} - \mathbf{D}_1)^{-1} & (\mathbf{I} - \mathbf{D}_1)^{-1}\mathbf{D}_2 \\
    \bm{0} & \mathbf{I}
    \end{pmatrix},
\label{eqn:fast_inverse}
\end{equation}
where $\mathbf{I} - \mathbf{D}_1$ is a regular matrix, and there exists an inverse matrix
because $\mathbf{D}_1$ is an upper triangular matrix, and all diagonal elements are zero.
The bottom two blocks do not need to be retained because they are not learned
and we can reduce the memory consumption. 
Since $N_{\text{in}}$ is, in general, $150$ to $4000$, 
the computation of the inverse matrix $(\mathbf{I} - \mathbf{D}_1)^{-1}$ is not expensive.
Thus, we can reduce the computational cost and memory consumption.

\subsection{Supervised Tree-Wasserstein Distancce}
Our goal is to construct a tree metric such that 
the tree-Wasserstein distance between documents with the same label is small and 
the tree-Wasserstein distance between documents with different labels is large.
To achieve this, we use a contrastive loss similar to prior works \cite{raia2006dimensionality} as follows:
\begin{dmath*}
    \mathcal{L}(\mathbf{D}_{\text{par}}, \mathbf{w}_v) = 
    \frac{1}{|\mathcal{D}_p|} \sum_{(i, j) \in \mathcal{D}_p} W^{\text{soft}}_{d_{\cal{T}}} (\mu_i, \mu_j) - \frac{1}{|\mathcal{D}_n|} \sum_{(i, j) \in \mathcal{D}_n} \!\!\min \left\{ W^{\text{soft}}_{d_{\cal{T}}} (\mu_i, \mu_j), m \right\},
\end{dmath*}
where $\mathbf{w}_v = (w_{v_1}, w_{v_2}, \cdots, w_{v_N})^\top$ is an $N$-dimensional vector,
$\mathcal{D}_p = \{(i, j) | y_i = y_j\}$ is a set of index pairs of documents that have the same label, 
$\mathcal{D}_n = \{(i, j) | y_i \not = y_j\}$ is a set of index pairs of documents that have different labels, and $m$ is the margin.

However, it is difficult to minimize this loss function with respect to $\mathbf{D}_{\text{par}}$ and $\mathbf{w}_v$
because the joint optimization of $\mathbf{D}_1$, $\mathbf{D}_2$, and $\mathbf{w}_v$ has too many degrees of freedom.
To solve this problem, 
we propose initializing $\mathbf{D}_1$ as an adjacency matrix of a tree consisting of internal nodes
and $\mathbf{w}_v=\mathbf{1}_N$,
fix $\mathbf{D}_1$ and $\mathbf{w}_v$ at the initial value,
and minimize the loss with respect to only $\mathbf{D}_2$.
In other words, given a tree $\mathcal{T}^{\prime} = (\bm{V}_{\text{in}}, \bm{E}_{\text{in}})$ whose adjacency matrix is $\mathbf{D}_1$
and edge lengths are all one, 
we optimize where to connect leaf nodes to $\mathcal{T}^{\prime}$.
As a by-product, the inverse matrix in Eq. (\ref{eqn:fast_inverse}) needs to be calculated only once before training. 
To optimize the loss function while satisfying the conditions of Theorem \ref{th:condition_of_adjacency_matrix}, 
we propose to calculate $\mathbf{D}_2$ using the softmax function as follows:
\begin{align*}
    [\mathbf{D}_2]_{i,j} &= \frac{\exp\left([\bm{\Theta}]_{i,j}\right)}{\sum_{i' = 1}^{N_{\text{in}}}\exp\left([\bm{\Theta}]_{i',j}\right)},
\end{align*}
where $\bm{\Theta} \in \mathbb{R}^{N_{\text{in}} \times N_{\text{leaf}}}$ is the parameter to be optimized.
Using the softmax function,
$\mathbf{D}_2^\top \mathbf{1}_{N_{\text{in}}} = \mathbf{1}_{N_{\text{leaf}}}$
and $\mathbf{D}_1$ is initialized such that $\mathbf{D}_1^\top \mathbf{1}_{N_{\text{in}}} = (0, 1, \cdots, 1)^\top$;
then $\mathbf{D}_1$ and $\mathbf{D}_2$ satisfy the conditions of Theorem \ref{th:condition_of_adjacency_matrix}.
Note that other softmax-like functions can also be used \cite{martins2016from, kong2019rankmax} 
as long as the constraint that the sum is one is satisfied.
In summary, our optimization problem is given as follows:
\begin{align}
\label{eqn:objective}
    \min_{\bm{\Theta} \in \mathbb{R}^{N_{\text{in}} \times N_{\text{leaf}}}} \mathcal{L}(\mathbf{D}_{\text{par}}, \mathbf{w}_v),
\end{align}
where $\mathbf{D}_1$ is fixed at initial values and $\mathbf{w}_v=\mathbf{1}_N$.
Since this objective function is differentiable with respect to $\mathbf{\Theta}$,
we can optimize it by stochastic gradient descent. 
After optimization,
for each leaf node, we select one of the most probable parents and construct the tree metric:
\begin{align*}
    \mathbf{D}_2^\ast & = (\mathbf{e}_1, \mathbf{e}_2, \ldots, \mathbf{e}_{N_{\text{leaf}}}) \in \{0, 1\}^{N_{\text{in}} \times N_{\text{leaf}}},
\end{align*}
where $\mathbf{e}_j \in \{0,1\}^{N_{\text{in}}}$ is the one-hot vector whose $k^\ast = \text{argmax}_{k} [\mathbf{D}_2]_{k,j}$ th element is one and the other elements are zero.
We substitute $\mathbf{D}_2^\ast$ and $\mathbf{D}_1$ in Eq. \eqref{eqn:block} and obtain the tree metric that represents the task-specific distance.
We refer to this approach as the \textit{Supervised Tree-Wasserstein} (STW) distance. 

The tree-Wasserstein distance between $\mu_i$ and $\mu_j$ can be considered as the L1 distance between their corresponding vectors.
Using the formulation of the soft tree-Wasserstein distance, the
tree-Wasserstein distance can be computed as the L1 norm of the following vector:
\begin{align*}
    \mathbf{w}_v
    \circ
    \left\{
    \begin{pmatrix}
    (\mathbf{I} - \mathbf{D}_1)^{-1} & (\mathbf{I} - \mathbf{D}_1)^{-1}\mathbf{D}_2^\ast \\
    \bm{0} & \mathbf{I}
    \end{pmatrix} 
    \left( 
    \begin{array}{c}
    \mathbf{0}_{N_{\text{in}}} \\
    \mathbf{a}_i - \mathbf{a}_j
    \end{array} \right) \right\},
\end{align*}
where $\circ$ is the element-wise Hadamard product.
As can be seen above, this formulation can be generalized to the case of comparing one document $\mathbf{a}_1$ with $M-1$ documents $\mathbf{a}_2, \mathbf{a}_3, \ldots, \mathbf{a}_M$.
Then $M-1$ documents can be compared simultaneously by replacing the right vector in the above equation with
$\left( \begin{array}{ccc}
    \mathbf{0}_{N_{\text{in}}} & \cdots & \mathbf{0}_{N_{\text{in}}} \\
    \mathbf{a}_2  -\mathbf{a}_1             & \cdots & \mathbf{a}_M - \mathbf{a}_1 
    \end{array} \right)$.
Therefore, the STW distance can be computed on a GPU and can compare multiple documents simultaneously.

\subsection{Implementation Details}
We initialize $\mathbf{D}_1$ such that the tree $\mathcal{T}^{\prime}$ with this adjacency matrix is a perfect $k$-ary tree of depth $d$.
We show the pseudo-code of the STW distance for inference in Algorithm \ref{alg:implementation}.
In practice, lines 4--7 need to be computed only once before inference.
During training, we skip line 6, 
use the approximation of the L1 norm in line 8,
compute the loss, and update the parameter $\mathbf{\Theta}$.
Since all operations can run on a GPU and are differentiable,
we can optimize $\mathbf{\Theta}$ using backpropagation and mini-batch stochastic
gradient descent.
This can be easily extended to an implementation that is suitable for batch processing.
We found that when the number of unique words contained in a document is large, the optimization is difficult because the elements of the normalized bag-of-words reach zero.
To address this issue, we multiply a fixed value $5$ to $\mathbf{a}$ in Algorithm \ref{alg:implementation} during training. 

For all $v_i \in \bm{V}_{\text{in}}$ and $v_j \in \bm{V}_{\text{leaf}}$,
the number of nodes contained in a path from $v_j$ to $v_i$ is at most $d+2$.
If a node $v_{j + N_{\text{in}}}$ is contained in the subtree rooted at $v_i$, 
then $[\mathbf{C}]_{i, j}$ is one, and is zero otherwise.
Therefore, $\mathbf{C}$ is a sparse matrix that has at most $(d+1) \times N_{\text{leaf}}$ non-zero elements,
and $\mathbf{a}$ is a sparse vector because $s \ll N_{\text{leaf}}$, 
where $s$ denotes the number of unique words contained in the two documents to be compared.
In general, since GPUs are not suitable for multiplications of sparse matrices,
it is faster to compute them as multiplications of dense matrices when computing on a GPU.
In the following experiments, we evaluate the STW distance on a GPU as multiplications of dense matrices.
However, when run on a CPU, it can be computed in $O(sd)$ by using this sparsity.

\begin{algorithm}[tb]
   \caption{Implementation of the STW distance, using PyTorch syntax.}
   \label{alg:implementation}
\begin{algorithmic}[1]
   \STATE {\bfseries Input:}  normalized bag-of-words $\mathbf{a}_i$, $\mathbf{a}_j$, $\mathbf{w}_v=\mathbf{1}_N$.
   \STATE {\bfseries Output:} tree-Wasserstein distance between $\mathbf{a}_i$ and $\mathbf{a}_j$.
   \STATE $\mathbf{a} = \mathbf{a}_i - \mathbf{a}_j$
   \STATE $\mathbf{A} = (\mathbf{I} - \mathbf{D}_1)^{-1}$
   \STATE $\mathbf{D}_2$ = softmax($\bm{\Theta}$, dim=0)
   \STATE $\mathbf{D}_2^\ast = \mathbf{D}_2.\text{ge(}\mathbf{D}_2\text{.max(}0, \text{keepdim=True)[0]).float()}$
   \STATE $\mathbf{C}$ = mm($\mathbf{A}$, $\mathbf{D}_2^\ast$)
   \STATE {\bfseries return} abs(mv($\mathbf{C}$, $\mathbf{a}$)).sum() + abs($\mathbf{a}$).sum()
\end{algorithmic}
\end{algorithm}
\section{Experimental Results}
We evaluate the following methods in document classification tasks on the synthetic and six real datasets following S-WMD
in the test error rate of the $k$-nearest neighbors ($k$NN) and the time consumption: 
\textsc{TWITTER}, \textsc{AMAZON}, \textsc{CLASSIC}, \textsc{BBCSPORT}, \textsc{OHSUMED}, and \textsc{REUTERS}.
Datasets are split into train/test as with the previous works \cite{kusner2015from, gao2016supervised}.
Table \ref{table:dataset} lists the number of unique words contained in the dataset (bag-of-words dimension) and the average number of unique words contained in a document for all real datasets.
\begin{table}[t!]
\vskip -0.15 in
\caption{Datasets used for the experiments.}
\label{table:dataset}
\vskip 0.1in
\begin{center}
\begin{small}
\begin{sc}
\begin{tabular}{lcc}
\toprule
         & bow dimension & average words \\
\midrule
TWITTER  &  6344   &   9.9  \\
CLASSIC  & 24277   &  38.6  \\
AMAZON   & 42063   &  45.0  \\
BBCSPORT & 13243   & 117    \\
OHSUMED  & 31789   &  59.2  \\
REUTERS  & 22425   &  37.1  \\
\bottomrule
\end{tabular}
\end{sc}
\end{small}
\end{center}
\vskip -0.2in
\end{table}

\subsection{Baseline Methods}
\textbf{Word Mover's Distance (WMD) \cite{kusner2015from}:} 
The document metric formulated by the optimal transport problem,
as described in Section \ref{sec:related_work}.

\textbf{Supervised Word Mover's Distance (S-WMD) \cite{gao2016supervised}:} 
Supervised metric learning based on WMD.

\textbf{Quadtree \cite{indyk2003fast}:}
To construct the tree metric,
we first obtain a randomly shifted hypercube containing all word embedding vectors. 
Next, we recursively divide the hypercube into hypercubes with half side length until there is only one word embedding vector in the hypercube. 
Each hypercube corresponds to a node, which has child nodes that correspond to hypercubes with half side length created by the split.
The tree constructed in this way is called Quadtree.
After constructing Quadtree,
we compute the tree-Wasserstein distance in Eq. \eqref{eqn:tree_wasserstein}.

\textbf{Flowtree \cite{arturs2020scalable}:}
Flowtree computes the transport plan on Quadtree, 
and then computes the cost on the ground metric.

\textbf{Tree-Sliced Wasserstein (TSW) Distance \cite{le2019tree}:} 
The TSW distance samples the tree metrics,
and then computes the average distance of tree-Wasserstein distances on these tree metrics.
A previous work \cite{le2019tree} showed that increasing the sampling size results in higher accuracy, but requires more computation time,
and recommended 10 samples.
Following this, we evaluated the TSW distance with the deepest level of the tree of 6 
and the number of child nodes of 5 with sampling numbers of 1, 5, and 10.
For sampling size, we refer to TSW-1, TSW-5, and TSW-10, respectively.

\textbf{Supervised Tree-Wasserstein (STW) Distance:}
We initialize $\mathbf{D}_1$ such that the tree whose adjacency matrix is $\mathbf{D}_1$ is a perfect 5-ary tree of depth 5,
and optimize Eq. (\ref{eqn:objective}) using Adam \cite{adam} and LARS \cite{LARS}.
After optimization, the deepest level of the tree is 5 or 6.
To select the margin $m$,
we use 20\% of the training dataset for validation.
We then train our model at a learning rate of $0.1$ and a batch size of $100$ for $30$ epochs. 
To avoid overfitting, we evaluated the STW distance using the parameters with the lowest loss in $30$ epochs of the validation dataset.

\subsection{Experimental Setup}
We use word2vec \cite{mikolov2013distributed}, which is pre-trained on Google News \footnote{\url{https://code.google.com/p/word2vec}}
as the word embedding vectors for WMD, S-WMD, Quadtree, Flowtree, and the TSW distance. 
For measuring the time consumption,
we use the public implementation \footnote{\url{https://github.com/mkusner/wmd}} of \cite{kusner2015from} for WMD
and the public implementation \footnote{\url{https://github.com/ilyaraz/ot_estimators}} of \cite{arturs2020scalable}, which is written in C++ and Python,
for Quadtree and Flowtree.
We implement S-WMD, and the TSW and STW distances in PyTorch.
The public implementation of WMD is written in C and Python and uses the algorithm developed by \cite{pele2009fast}, which requires cubic time.
Additionally, we implement WMD with Sinkhorn algorithm in PyTorch, which we refer to as WMD (Sinkhorn).
The parameter of the Sinkhorn algorithm for WMD (Sinkhorn) and our implementation of S-WMD 
is same as the public implementation \footnote{\url{https://github.com/gaohuang/S-WMD}} of \cite{gao2016supervised}.
We evaluated WMD (Sinkhorn), S-WMD, and the TSW and STW distances on Nvidia Quadro RTX 8000,
and WMD, Quadtree, and Flowtree on Intel Xeon CPU E5-2690 v4 (2.60 GHz).

\subsection{Results on the Synthetic Dataset}
\begin{figure}[t]
\vskip -0.0in
\subfigure[Quadtree]{
\includegraphics[width=0.3\columnwidth]{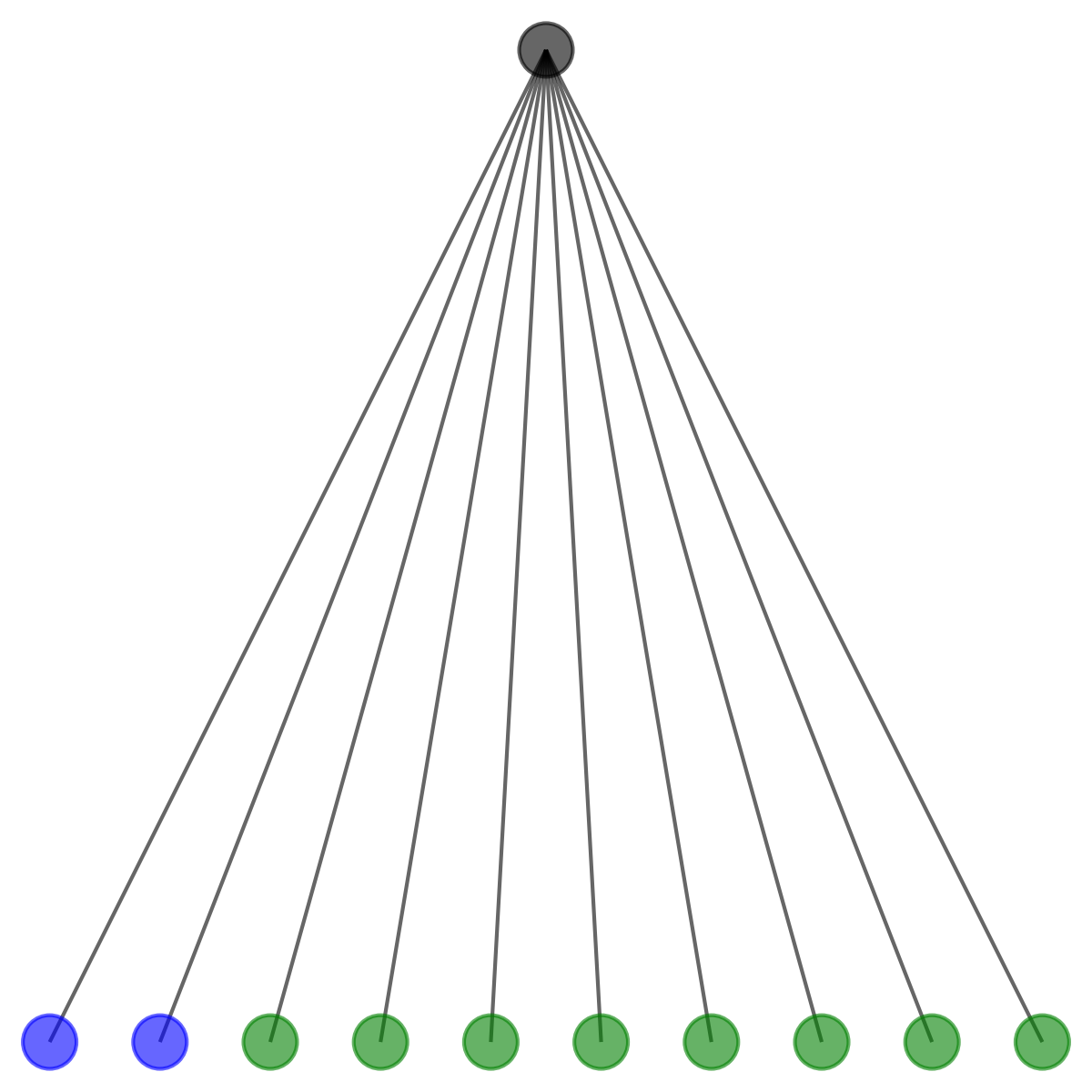}
}
\subfigure[TSW]{
\includegraphics[width=0.3\columnwidth]{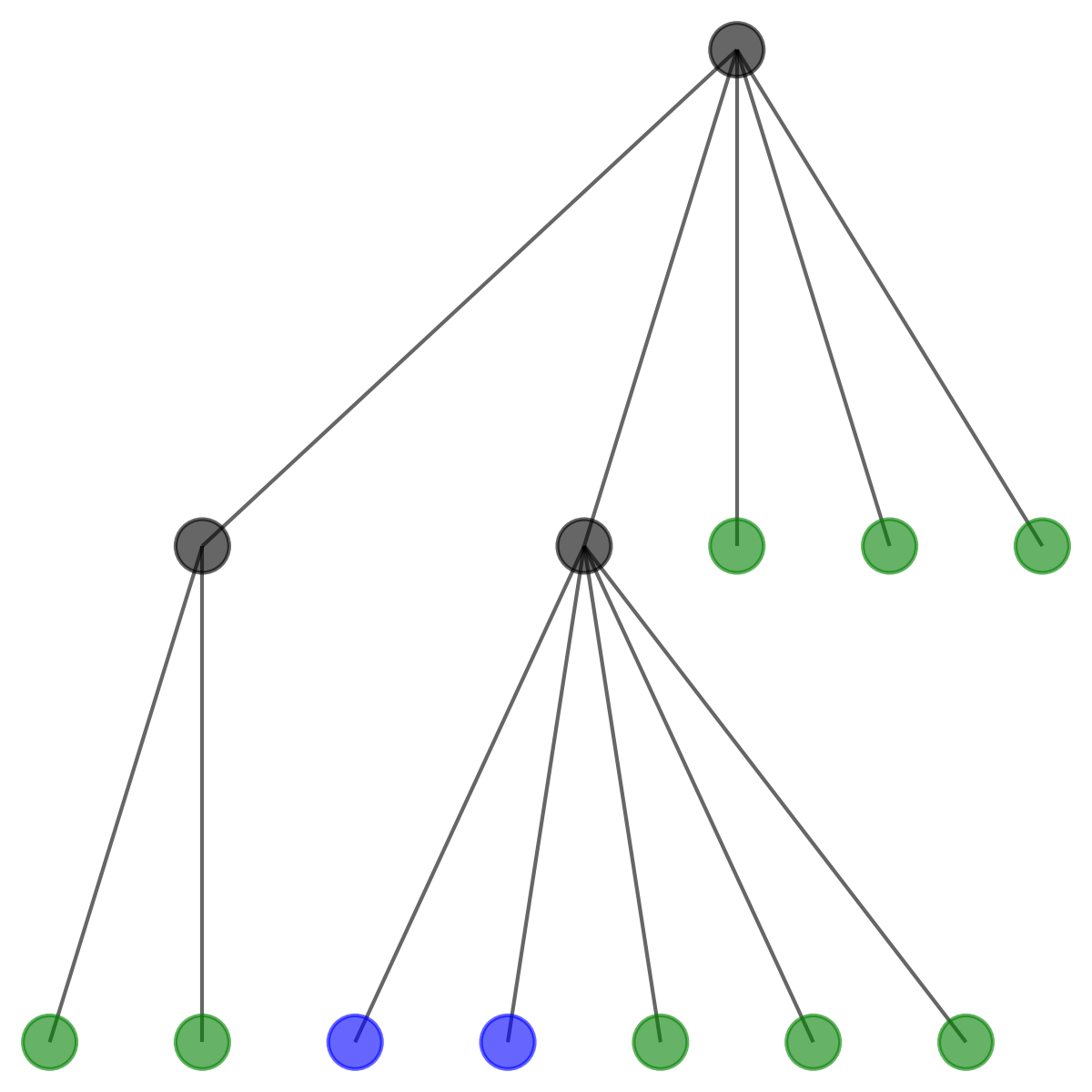}
}
\subfigure[STW]{
\includegraphics[width=0.3\columnwidth]{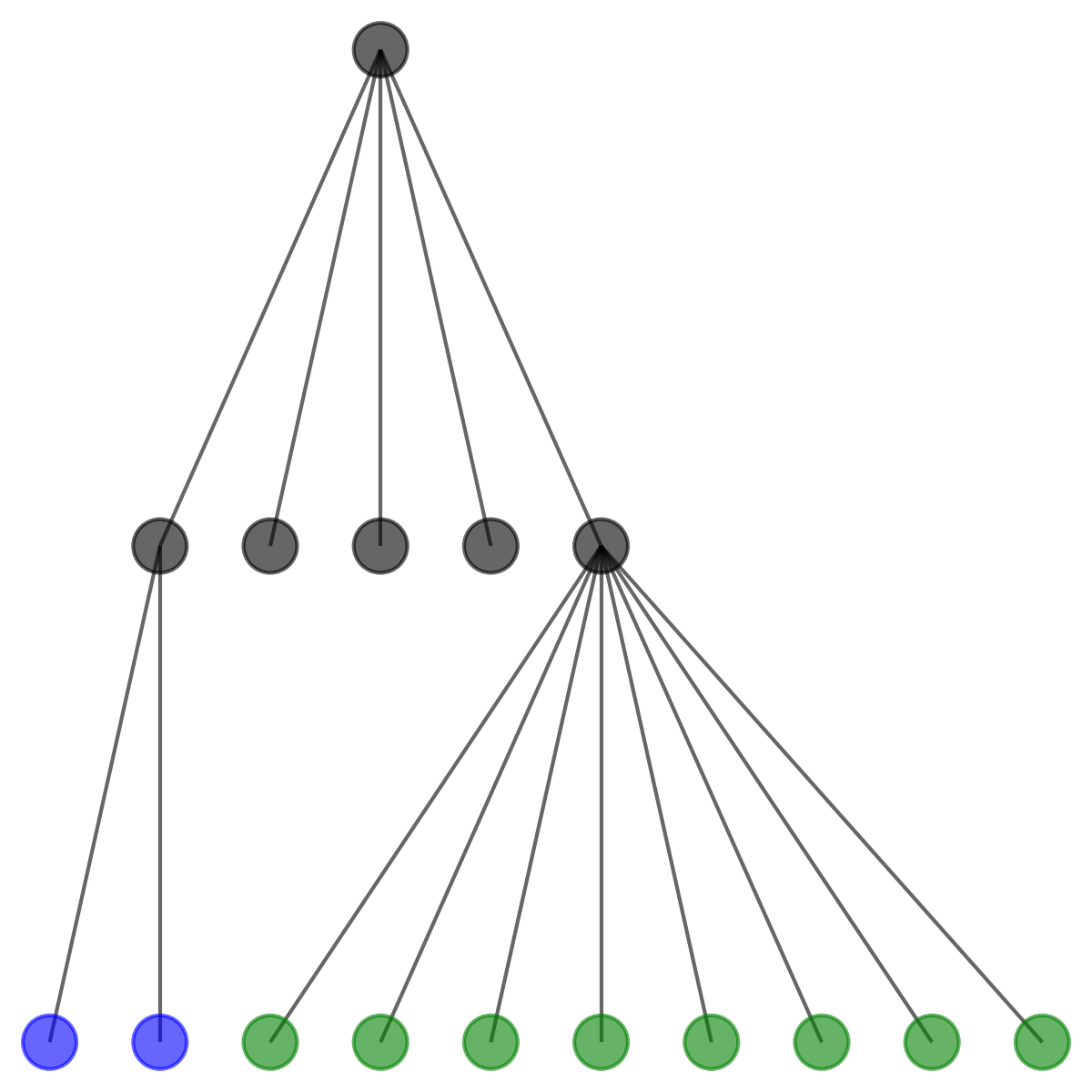}
}
\begin{center}
\vskip -0.15 in
\caption{Trees constructed by Quadtree, the TSW distance, and the STW distance on the synthetic dataset. Flowtree computes the optimal flow on Quadtree. Nodes that correspond to internal nodes are black-filled; nodes that correspond to the words ``piano'' and ``violin'' are blue-filled; and others are green-filled.}
\label{fig:visualization_of_synthetic}
\end{center}
\vskip -0.3in
\end{figure}
\begin{table}[t]
\caption{The $k$NN test error rate on the synthetic dataset.}
\label{table:accuracy_on_syntheric}
\vskip 0.1in
\begin{center}
\begin{small}
\begin{sc}
\begin{tabular}{cccc}
\toprule
Quadtree & Flowtree & TSW-1/5/10 & STW \\
\midrule
0.3 & 1.6 & 7.5 / 4.2 / 3.9 & 0.0 \\
\bottomrule
\end{tabular}
\end{sc}
\end{small}
\end{center}
\vskip -0.2in
\end{table}
\begin{table*}[t!]
\vskip -0.1 in
\caption{The $k$NN test error for real datasets.
         WMD and S-WMD give the results from \cite{gao2016supervised}.}
\label{table:accuracy}
\vskip 0.1in
\begin{center}
\begin{small}
\begin{sc}
\begin{tabular}{lcccccc}
\toprule
                  & TWITTER        & AMAZON        &       CLASSIC &      BBCSPORT & OHSUMED & REUTERS \\
\midrule
WMD               & 28.7 $\pm$ 0.6 &  7.4 $\pm$ 0.3 & 2.8 $\pm$ 0.1 &  4.6 $\pm$ 0.7 &    44.5 & 3.5 \\
S-WMD             & 27.5 $\pm$ 0.5 &  5.8 $\pm$ 0.1 & 3.2 $\pm$ 0.2 &  2.1 $\pm$ 0.5 &    34.3 & 3.2 \\
Quadtree          & 30.4 $\pm$ 0.8 & 10.7 $\pm$ 0.3 & 4.1 $\pm$ 0.4 &  4.5 $\pm$ 0.5 &    44.0 & 5.2 \\
Flowtree          & 29.8 $\pm$ 0.9 &  9.9 $\pm$ 0.3 & 5.6 $\pm$ 0.6 &  4.7 $\pm$ 1.1 &    44.4 & 4.7 \\
TSW-1             & 30.2 $\pm$ 1.3 & 14.5 $\pm$ 0.6 & 5.5 $\pm$ 0.5 & 12.4 $\pm$ 1.9 &    58.4 & 7.5 \\
TSW-5             & 29.5 $\pm$ 1.1 &  9.2 $\pm$ 0.1 & 4.1 $\pm$ 0.4 & 11.9 $\pm$ 1.3 &    51.7 & 5.8 \\
TSW-10            & 29.3 $\pm$ 1.0 &  8.9 $\pm$ 0.5 & 4.1 $\pm$ 0.6 & 11.4 $\pm$ 0.9 &    51.1 & 5.4 \\
STW               & 28.9 $\pm$ 0.7 & 10.1 $\pm$ 0.7 & 4.4 $\pm$ 0.7 &  3.4 $\pm$ 0.8 &    40.2 & 4.4 \\
\bottomrule
\end{tabular}
\end{sc}
\end{small}
\end{center}
\vskip -0.1in
\end{table*}
\begin{figure*}[t!]
\vskip 0.0 in
\begin{center}
\centerline{\includegraphics[width=0.9\hsize]{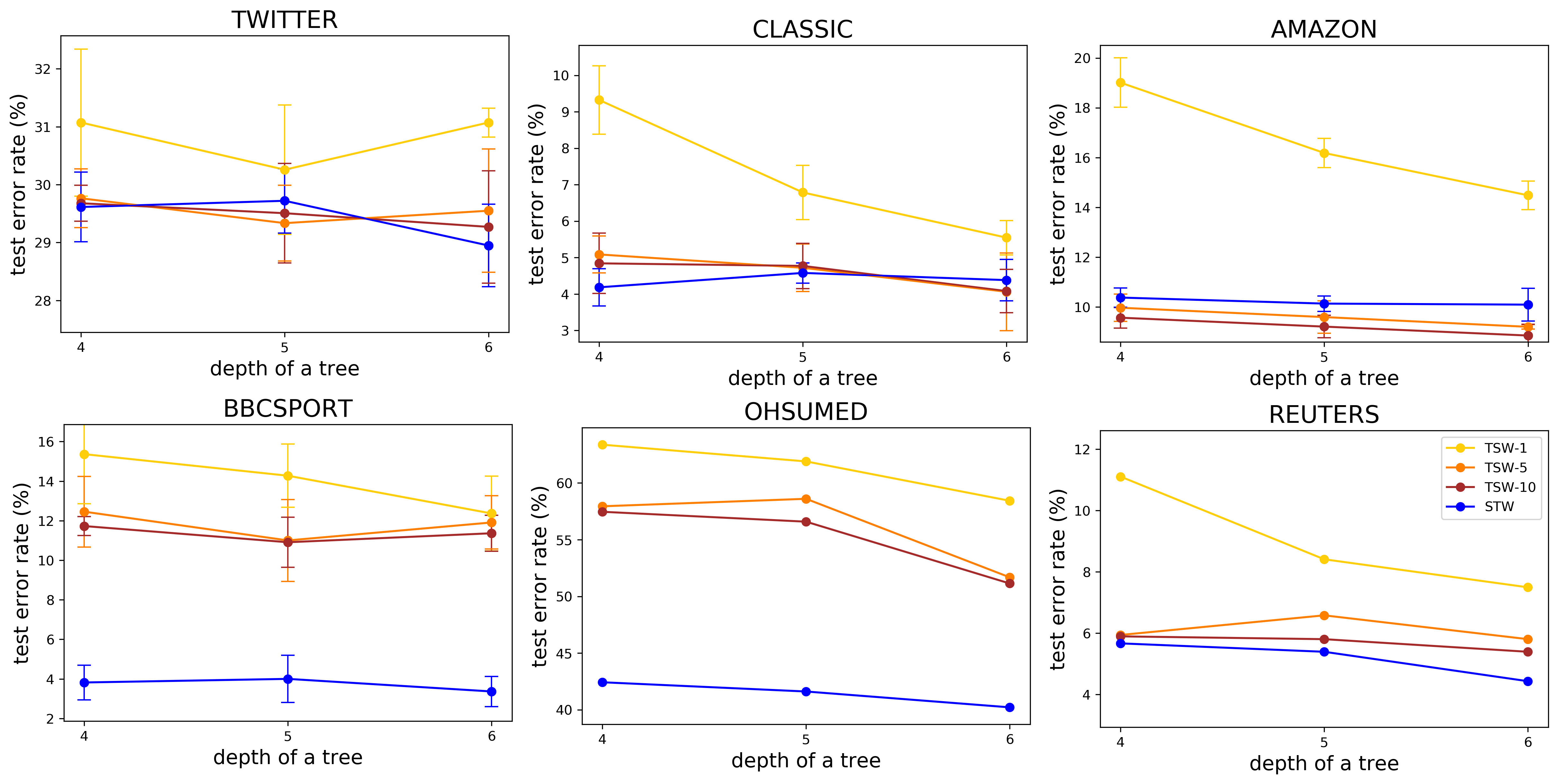}}
\vskip -0.1 in
\caption{The $k$NN test error rate on real datasets when varying the depth level of the tree. For the STW distance, if the tree consisting of internal nodes is initialized so that its depth is $d$, the depth of the tree after optimization is $d$ or $d+1$. In this figure, when the depth of the tree consisting of internal nodes is initialized such that its depth is $d$, the depth of the STW distance is considered to be $d+1$.}
\label{fig:depth}
\end{center}
\vskip -0.3in
\end{figure*}
By using the synthetic dataset, we first show that the STW distance can construct a tree metric that represents a task-specific distance 
and improves the accuracy of the document classification task.
We generated the synthetic dataset so that documents consist of only ten words: ``piano,” ``violin,” ``cello,” ``viola,” ``contrabass,” ``trumpet,” ``trombone,” ``clarinet,” ``flute,” and ``harpsichord.” Each word contains zero or one
and documents are classified into two classes based on whether the word ``piano'' or ``violin'' is contained.
We initialize $\mathbf{D}_1$ so that the tree whose adjacency matrix is $\mathbf{D}_1$ is a perfect 5-ary tree of depth 1 for easy visualization.

We show the trees constructed by Quadtree, Flowtree, the TSW and STW distances in Figure \ref{fig:visualization_of_synthetic}
and the $k$NN test error rate in Table \ref{table:accuracy_on_syntheric}.
Quadtree constructs a tree so that the distance between all words is the same because the dimension of the word embedding vector is high and each word is assigned to a different hypercube.
The TSW distance constructs a tree so that the words ``piano'' and ``violin'' are not far from other words.
However, the STW distance constructs a tree so that the words ``piano'' and ``violin'' are close and far from other words,
and the words except for the words ``piano'' and ``violin'' are close together.
As a result, the STW distance outperforms Quadtree, Flowtree, and the TSW distance.

\subsection{Results on Real Datasets}
\begin{figure*}[t!]
\vskip 0.0 in
\begin{center}
\centerline{\includegraphics[width=\hsize]{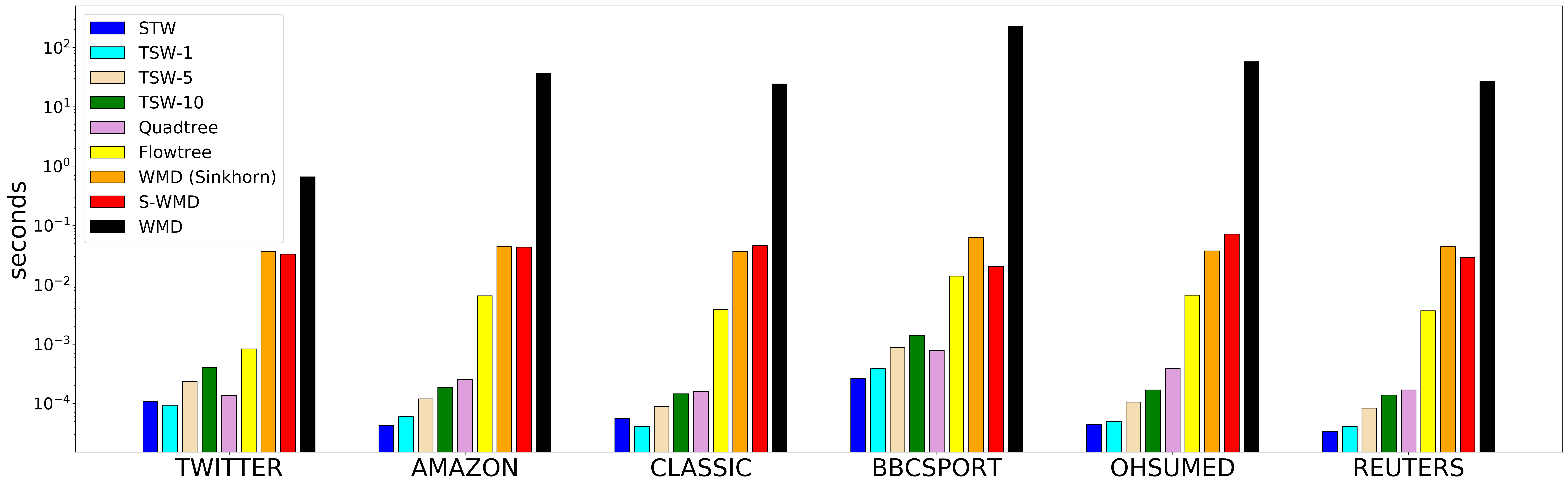}}
\vskip -0.15 in
\caption{Average time consumption for comparing $500$ documents with one document. For the STW distance and the TSW distance, the batch size is set to the number of documents contained in the training dataset. For WMD (Sinkhorn) and S-WMD, the batch size is set to $500$ due to the memory size limitations. To obtain the average time consumption, we sample $100$ documents as queries and measure the time consumption.}
\label{fig:speed}
\end{center}
\vskip -0.3in
\end{figure*}
\begin{figure}[t]
\vskip -0.0in
\begin{center}
\centerline{\includegraphics[width=\columnwidth]{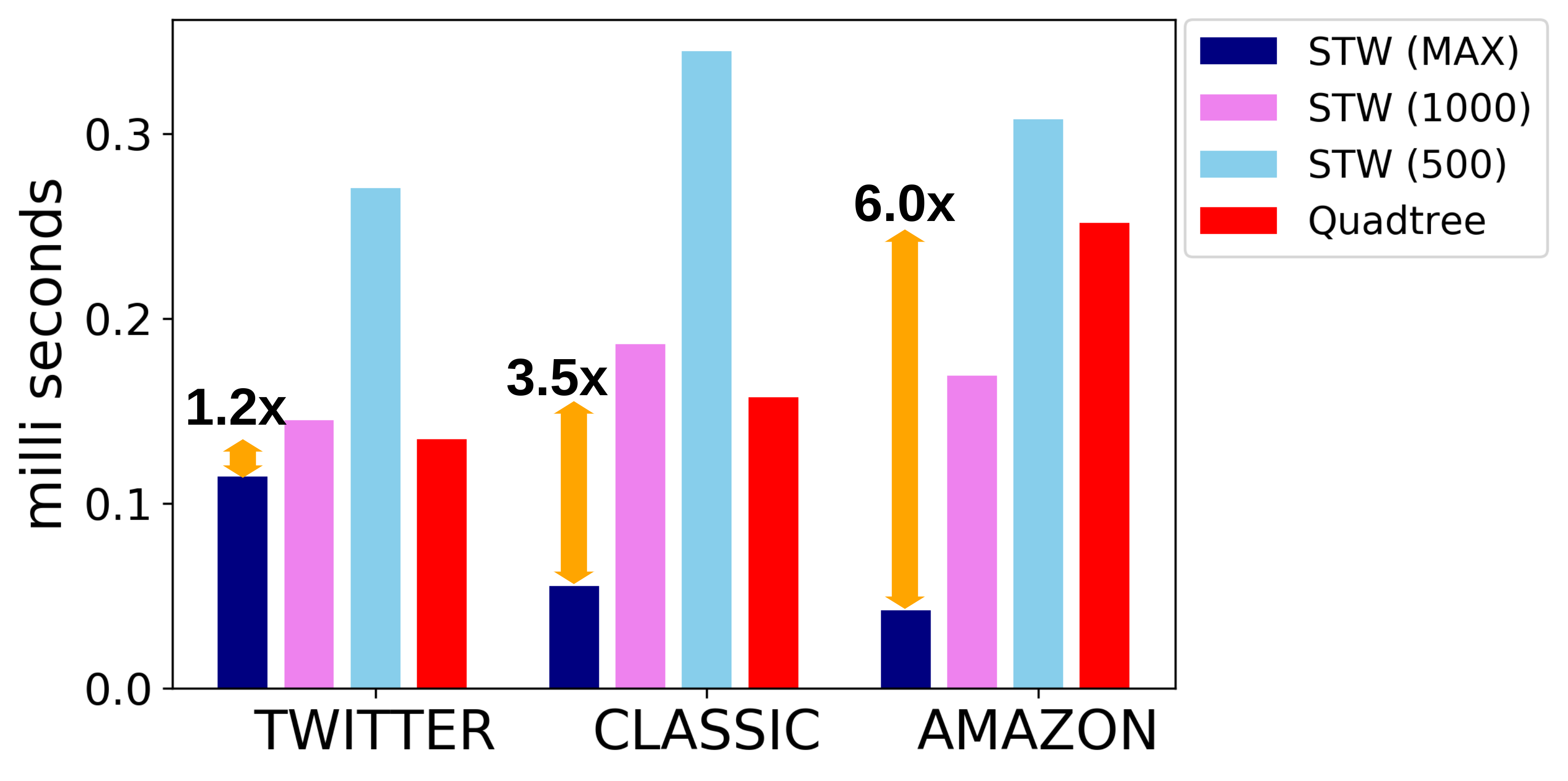}}
\vskip -0.15 in
\caption{Average time consumption to compare one document with 500 documents.
The number in the bracket indicates the batch size 
and \textit{MAX} means the number of documents contained in the training dataset.}
\label{fig:speed_batch}
\end{center}
\vskip -0.4in
\end{figure}
We first discuss the accuracy of document classification tasks on real datasets, and then discuss the time consumption to compute the distances.
We list the $k$NN test error rates in Table \ref{table:accuracy}.
On TWITTER, BBCSPORT, OHSUMED, and REUTERS,
the STW distance outperforms Quadtree, Flowtree, and the TSW distance.
On AMAZON and CLASSIC,
the STW distance outperforms the TSW-1 distance and is competitive with Quadtree, Flowtree, the TSW-5 distance, and the TSW-10 distance, respectively.
In particular, the error rate of the TSW distance is approximately $10\%$ higher than that of WMD on BBCSPORT and OHSUMED,
but the STW distance improves the error rate and outperforms WMD. 
On the other hand, the STW distance still underperforms WMD in other datasets
and all tree-based methods underperform S-WMD in all datasets.

To construct the tree metric in the TSW and STW distances,
we need to set the depth level of the tree as the hyperparameters.
We evaluated how the tree's depth level affects the accuracy of the TSW and STW distances.
In Figure \ref{fig:depth}, we show the $k$NN test error rate
when the STW distance is initialized, such that $\mathbf{D}_1$ is an adjacency matrix of the depth level of trees 3, 4, and 5,
and the TSW distance is sampled so that the depth level of the tree is 4, 5, and 6.
The results show that, in general, 
the deeper the depth level of the tree, the higher the accuracy.
When the depth level of the tree is 4, 
the accuracy of the TSW-1 distance is considerably worse than when the depth level of the tree is 6, 
whereas the STW distance is only approximately $2\%$ worse.
The results indicate that the STW distance is more accurate than the TSW-1 distance, especially when the tree is shallow.

Next, we discuss the average time consumption to calculate distance.
We show the time required to compare $500$ documents with one document in Figure \ref{fig:speed}.
Quadtree, Flowtree, and the TSW and STW distances are faster than WMD, WMD (Sinkhorn), and S-WMD on all datasets.
The TSW-10 distance calculates the tree-Wasserstein distance 10 times, which is approximately 10 times slower than Quadtree, and the TSW-1 and STW distances.
The public implementation of Quadtree uses an algorithm that is suitable for CPUs, which runs in linear time with respect to the number of unique words in the document. 
The time complexity of the implementation of the STW distance 
depends on the number of unique words in the dataset, 
but runs on a GPU and is suitable for batch processing.
Therefore, when comparing a large number of documents, our algorithm is more efficient than the existing algorithm for computing the tree-Wasserstein distance.
In Figure \ref{fig:speed_batch}, we show the average time consumption when varying the batch sizes on TWITTER, CLASSIC, and AMAZON for Quadtree and the STW distance.
The results indicate that, if the batch size is sufficiently large, the STW distance is faster than Quadtree. 
In particular, on AMAZON, when the batch size is set to the number of documents contained in the training dataset,
the STW distance is about six times faster than Quadtree.
Additional experiments when varying the batch size are included in the Appendix. 

\section{Conclusion}
In this work, we proposed the soft tree-Wasserstein distance and the supervised tree-Wasserstein distance.
The soft tree-Wasserstein distance is differentiable with respect to the probability of the parent--child relationships of a tree and is formulated only by matrix multiplications.
By using the soft tree-Wasserstein distance, we formulated the STW distance as a continuous optimization problem, 
which is end-to-end trainable and constructs the tree metric by leveraging the label information of documents.
Through the experiments on the synthetic and real datasets,
we showed that the STW distance can be computed quickly and can improve the accuracy of document classification tasks.
Furthermore, 
because the STW distance is suitable for batch processing,
it is more efficient than existing methods for computing the Wasserstein distance, especially when comparing a large number of documents.

\section*{Acknowledgement}
We thank Hisashi Kashima and Shogo Hayashi for their useful discussions.
M.Y. was supported by MEXT KAKENHI 20H04243.

\bibliography{example_paper}

\begin{thebibliography}{30}
\providecommand{\natexlab}[1]{#1}
\providecommand{\url}[1]{\texttt{#1}}
\expandafter\ifx\csname urlstyle\endcsname\relax
  \providecommand{\doi}[1]{doi: #1}\else
  \providecommand{\doi}{doi: \begingroup \urlstyle{rm}\Url}\fi

\bibitem[Atasu \& Mittelholzer(2019)Atasu and Mittelholzer]{kubilay2019linear}
Atasu, K. and Mittelholzer, T.
\newblock Linear-complexity data-parallel earth mover’s distance
  approximations.
\newblock In \emph{International Conference on Machine Learning}, 2019.

\bibitem[Backurs et~al.(2020)Backurs, Dong, Indyk, Razenshteyn, and
  Wagner]{arturs2020scalable}
Backurs, A., Dong, Y., Indyk, P., Razenshteyn, I., and Wagner, T.
\newblock Scalable nearest neighbor search for optimal transport.
\newblock In \emph{International Conference on Machine Learning}, 2020.

\bibitem[Chami et~al.(2020)Chami, Gu, Chatziafratis, and Re]{chami2020from}
Chami, I., Gu, A., Chatziafratis, V., and Re, C.
\newblock From trees to continuous embeddings and back: Hyperbolic hierarchical
  clustering.
\newblock In \emph{Advances in Neural Information Processing Systems}, 2020.

\bibitem[Cuturi(2013)]{cuturi2013sinkhorn}
Cuturi, M.
\newblock Sinkhorn distances: Lightspeed computation of optimal transport.
\newblock In \emph{Advances in Neural Information Processing Systems}, 2013.

\bibitem[Dasgupta(2016)]{dasgupta2016cost}
Dasgupta, S.
\newblock A cost function for similarity-based hierarchical clustering.
\newblock In \emph{ACM Symposium on Theory of Computing}, 2016.

\bibitem[Deshpande et~al.(2019)Deshpande, Hu, Sun, Pyrros, Siddiqui, Koyejo,
  Zhao, Forsyth, and Schwing]{deshpande2019max}
Deshpande, I., Hu, Y.-T., Sun, R., Pyrros, A., Siddiqui, N., Koyejo, S., Zhao,
  Z., Forsyth, D., and Schwing, A.~G.
\newblock Max-sliced wasserstein distance and its use for gans.
\newblock In \emph{IEEE conference on Computer Vision and Pattern Recognition},
  2019.

\bibitem[Ganea et~al.(2018)Ganea, Becigneul, and Hofmann]{ganea2018hyperbolic}
Ganea, O., Becigneul, G., and Hofmann, T.
\newblock Hyperbolic entailment cones for learning hierarchical embeddings.
\newblock In \emph{International Conference on Machine Learning}, 2018.

\bibitem[Hadsell et~al.(2006)Hadsell, Chopra, and
  LeCun]{raia2006dimensionality}
Hadsell, R., Chopra, S., and LeCun, Y.
\newblock Dimensionality reduction by learning an invariant mapping.
\newblock In \emph{IEEE conference on Computer Vision and Pattern Recognition},
  2006.

\bibitem[Huang et~al.(2016)Huang, Guo, Kusner, Sun, Sha, and
  Weinberger]{gao2016supervised}
Huang, G., Guo, C., Kusner, M.~J., Sun, Y., Sha, F., and Weinberger, K.~Q.
\newblock Supervised word mover\textquotesingle s distance.
\newblock In \emph{Advances in Neural Information Processing Systems}, 2016.

\bibitem[Indyk \& Thaper(2003)Indyk and Thaper]{indyk2003fast}
Indyk, P. and Thaper, N.
\newblock Fast image retrieval via embeddings.
\newblock In \emph{International Workshop on Statistical and Computational
  Theories of Vision}, 2003.

\bibitem[Kingma \& Ba(2015)Kingma and Ba]{adam}
Kingma, D.~P. and Ba, J.
\newblock Adam: {A} method for stochastic optimization.
\newblock In \emph{International Conference on Learning Representations}, 2015.

\bibitem[Kolouri et~al.(2018)Kolouri, Rohde, and Hoffmann]{kolouri2018sliced}
Kolouri, S., Rohde, G.~K., and Hoffmann, H.
\newblock Sliced wasserstein distance for learning gaussian mixture models.
\newblock In \emph{IEEE conference on Computer Vision and Pattern Recognition},
  2018.

\bibitem[Kolouri et~al.(2019{\natexlab{a}})Kolouri, Nadjahi, Simsekli, Badeau,
  and Rohde]{kolouri2019generalized}
Kolouri, S., Nadjahi, K., Simsekli, U., Badeau, R., and Rohde, G.
\newblock Generalized sliced wasserstein distances.
\newblock In \emph{Advances in Neural Information Processing Systems},
  2019{\natexlab{a}}.

\bibitem[Kolouri et~al.(2019{\natexlab{b}})Kolouri, Pope, Martin, and
  Rohde]{kolouri2019sliced}
Kolouri, S., Pope, P.~E., Martin, C.~E., and Rohde, G.~K.
\newblock Sliced wasserstein auto-encoders.
\newblock In \emph{International Conference on Learning Representations},
  2019{\natexlab{b}}.

\bibitem[Kong et~al.(2020)Kong, Krichene, Mayoraz, Rendle, and
  Zhang]{kong2019rankmax}
Kong, W., Krichene, W., Mayoraz, N., Rendle, S., and Zhang, L.
\newblock Rankmax: An adaptive projection alternative to the softmax function.
\newblock In \emph{Advances in Neural Information Processing Systems}, 2020.

\bibitem[Korte \& Vygen(2006)Korte and Vygen]{korte2007combinatorial}
Korte, B. and Vygen, J.
\newblock \emph{Combinatorial Optimization: Theory and Algorithms}.
\newblock Springer, 3rd edition, 2006.

\bibitem[Kusner et~al.(2015)Kusner, Sun, Kolkin, and
  Weinberger]{kusner2015from}
Kusner, M.~J., Sun, Y., Kolkin, N.~I., and Weinberger, K.~Q.
\newblock From word embeddings to document distances.
\newblock In \emph{International Conference on Machine Learning}, 2015.

\bibitem[Lange et~al.(2014)Lange, Z{\"{u}}hlke, Holz, and Villmann]{smoothabs}
Lange, M., Z{\"{u}}hlke, D., Holz, O., and Villmann, T.
\newblock Applications of lp-norms and their smooth approximations for gradient
  based learning vector quantization.
\newblock In \emph{European Symposium on Artificial Neural Networks}, 2014.

\bibitem[Le et~al.(2019)Le, Yamada, Fukumizu, and Cuturi]{le2019tree}
Le, T., Yamada, M., Fukumizu, K., and Cuturi, M.
\newblock Tree-sliced variants of wasserstein distances.
\newblock In \emph{Advances in Neural Information Processing Systems}, 2019.

\bibitem[Liu et~al.(2020)Liu, Zhu, Yamada, and Yang]{yabin2020semantic}
Liu, Y., Zhu, L., Yamada, M., and Yang, Y.
\newblock Semantic correspondence as an optimal transport problem.
\newblock In \emph{IEEE conference on Computer Vision and Pattern Recognition},
  2020.

\bibitem[Martins \& Astudillo(2016)Martins and Astudillo]{martins2016from}
Martins, A. and Astudillo, R.
\newblock From softmax to sparsemax: A sparse model of attention and
  multi-label classification.
\newblock In \emph{International Conference on Machine Learning}, 2016.

\bibitem[Mikolov et~al.(2013)Mikolov, Sutskever, Chen, Corrado, and
  Dean]{mikolov2013distributed}
Mikolov, T., Sutskever, I., Chen, K., Corrado, G.~S., and Dean, J.
\newblock Distributed representations of words and phrases and their
  compositionality.
\newblock In \emph{Advances in Neural Information Processing Systems}, 2013.

\bibitem[Monath et~al.(2019)Monath, Zaheer, Silva, McCallum, and
  Ahmed]{monath2019gradient}
Monath, N., Zaheer, M., Silva, D., McCallum, A., and Ahmed, A.
\newblock Gradient-based hierarchical clustering using continuous
  representations of trees in hyperbolic space.
\newblock In \emph{International Conference on Knowledge Discovery and Data
  Mining}, 2019.

\bibitem[Nickel \& Kiela(2017)Nickel and Kiela]{nickel2017poincare}
Nickel, M. and Kiela, D.
\newblock Poincar\'{e} embeddings for learning hierarchical representations.
\newblock In \emph{Advances in Neural Information Processing Systems}, 2017.

\bibitem[Pele \& Werman(2009)Pele and Werman]{pele2009fast}
Pele, O. and Werman, M.
\newblock Fast and robust earth mover's distances.
\newblock In \emph{IEEE conference on International Conference on Computer
  Vision}, 2009.

\bibitem[Rabin et~al.(2011)Rabin, Peyr{\'{e}}, Delon, and
  Bernot]{rabin2011wasserstein}
Rabin, J., Peyr{\'{e}}, G., Delon, J., and Bernot, M.
\newblock Wasserstein barycenter and its application to texture mixing.
\newblock In \emph{Scale Space and Variational Methods in Computer Vision},
  2011.

\bibitem[Sarlin et~al.(2020)Sarlin, DeTone, Malisiewicz, and
  Rabinovich]{sarlin2020superGlue}
Sarlin, P., DeTone, D., Malisiewicz, T., and Rabinovich, A.
\newblock Superglue: Learning feature matching with graph neural networks.
\newblock In \emph{IEEE conference on Computer Vision and Pattern Recognition},
  2020.

\bibitem[Sato et~al.(2020)Sato, Yamada, and Kashima]{sato2019unbalanced}
Sato, R., Yamada, M., and Kashima, H.
\newblock Fast unbalanced optimal transport on a tree.
\newblock In \emph{Advances in Neural Information Processing Systems}, 2020.

\bibitem[You et~al.(2017)You, Gitman, and Ginsburg]{LARS}
You, Y., Gitman, I., and Ginsburg, B.
\newblock Large batch training of convolutional networks.
\newblock \emph{arXiv preprint arXiv:1708.03888}, 2017.

\bibitem[Yurochkin et~al.(2019)Yurochkin, Claici, Chien, Mirzazadeh, and
  Solomon]{yurochkin2019hierarchical}
Yurochkin, M., Claici, S., Chien, E., Mirzazadeh, F., and Solomon, J.~M.
\newblock Hierarchical optimal transport for document representation.
\newblock In \emph{Advances in Neural Information Processing Systems}, 2019.

\end{thebibliography}
\bibliographystyle{icml2021}

\clearpage
\appendix
\section{Proofs}
\subsection{Proof of Theorem \ref{th:condition_of_adjacency_matrix}}
To prove Theorem \ref{th:condition_of_adjacency_matrix},
we show the theorems presented in \cite{korte2007combinatorial}.
The number of theorems in the bracket is the number of theorems in \cite{korte2007combinatorial}.
\begin{theorem}[Theorem 2.5]
Let $\bm{V} = \{v_1, \cdots, v_N\}$ be a set of nodes and $G = (\bm{V}, \bm{E})$ be a directed graph.
Then, the following statements are equivalent:
\begin{itemize}
    \item $G$ is a directed tree with the root $v_1$.
    \item For all $v \in \bm{V}$, $(v_1, v) \not \in \bm{E}$, and for all $v \in \bm{V} \setminus \{v_1\}$, a unique $u \in V$ exists so that $(v, u) \in \bm{E}$, and $G$ contain no circuit.
\end{itemize}
\label{th:directed_tree}
\end{theorem}
\begin{definition}[Definition 2.8]
Let $\bm{V} = \{v_1, \ldots, v_N\}$ be a set of nodes and $G = (\bm{V}, \bm{E})$ be a directed graph.
A \textit{topological order} of $G$ is an order of the nodes 
so that for each edge $(v_i, v_j) \in \bm{E}$, we have $i < j$.
\end{definition}
\begin{theorem}[Proposition 2.9]
A directed graph has a topological order if and only if it is acyclic.
\end{theorem}
By replacing all edges $(v_i, v_j) \in \bm{E}$ with $(v_j, v_i)$,
we have the following.
\begin{corollary}
Let $\bm{V} = \{v_1, \ldots, v_N\}$ be a set of nodes and $G = (\bm{V}, \bm{E})$ be a directed graph.
If $i > j$ for all edges $(v_i, v_j) \in \bm{E}$, then $G$ is acyclic.
\label{cor:acyclic}
\end{corollary}
By using these theorems, we prove Theorem \ref{th:condition_of_adjacency_matrix}.
\begin{proof}
Because the adjacency matrix $\mathbf{D}_{\text{par}}$ satisfies condition (\ref{cond:eqn}) in Theorem \ref{th:condition_of_adjacency_matrix},
for all $v \in \bm{V}$, we have $(v_1, v) \not \in \bm{E}$, 
and for all $v \in \bm{V} \setminus \{v_1\}$, 
there exists a unique $u \in V$ such that $(v, u) \in \bm{E}$.
Because the adjacency matrix $\mathbf{D}_{\text{par}}$ satisfies conditions (\ref{cond:strictly_upper}) in Theorem \ref{th:condition_of_adjacency_matrix},
we have $i > j$ for all edges $(v_i, v_j) \in \bm{E}$.
Due to Corollary \ref{cor:acyclic}, $G$ is acyclic.

Therefore, $G$ is a directed tree with root $v_1$ by Theorem \ref{th:directed_tree}.
\end{proof}

\subsection{Details of Eq. \eqref{eqn:psub}}
Because $\mathbf{D}_{\text{par}}$ is a nilpotent matrix,
and $\mathbf{D}_{\text{par}}^N$ is a zero matrix, 
\begin{align*}
    (\mathbf{I} - \mathbf{D}_{\text{par}}) \sum_{k=0}^{\infty} \mathbf{D}_{\text{par}}^k &= (\mathbf{I} - \mathbf{D}_{\text{par}}) \sum_{k=0}^{N-1} \mathbf{D}_{\text{par}}^k \\
    &= \mathbf{I} - \mathbf{D}_{\text{par}}^{N} \\
    &= \mathbf{I}.
\end{align*}
Because $\mathbf{I} - \mathbf{D}_{\text{par}}$ is an upper triangular matrix and all diagonal elements are one,
$\mathbf{I} - \mathbf{D}_{\text{par}}$ is a regular matrix.
Therefore, the sum of the infinite geometric series converges to $(\mathbf{I} - \mathbf{D}_{\text{par}})^{-1}$.

\subsection{Proof of Theorem \ref{th:soft}}
\begin{proof}
Assume that the tree metric is given, and 
let $\mathbf{D}_{\text{par}}$ be its adjacency matrix.
The element in the $i$-th row and $j$-th column of the adjacency matrix to the power of $k$ is
the number of paths from $v_j$ to $v_i$ with $k$ steps.
$\mathbf{D}_{\text{par}}$ is the adjacency matrix of a tree, 
and the number of paths is at most $1$.
If there is a path from $v_j$ to $v_i$ with $k$ steps, $[\mathbf{D}_{\text{par}}^k]_{i,j}$ is one; 
otherwise, it is zero.
Then if there is a path from $v_j$ to $v_i$, 
$[(\mathbf{I} - \mathbf{D}_{\text{par}})^{-1}]_{i, j}$ is one; otherwise, it is zero.
The existence of a path from $v_j$ to $v_i$ means that $v_j$ is contained in the subtree rooted at $v_i$.
From the definition of $P_{\text{sub}}(v_j | v_i)$,
if $v_j$ is contained in the subtree rooted at $v_i$, 
$P_{\text{sub}}(v_j | v_i)$ is one; otherwise, it is zero.
We now have
\begin{align*}
    \mu(\Gamma(v)) = \sum_{u \in \Gamma(v)} \mu(u) = \sum_{u \in \bm{V}_{\text{leaf}}} P_{\text{sub}}(u | v) \mu(u).
\end{align*}
Therefore, if the tree metric is given and $\alpha$ approaches $\infty$,
the soft tree-Wasserstein distance converges to the tree-Wasserstein distance that is,
\begin{align*}
    W_{d_{\mathcal{T}}}^{\text{soft}} (\mu_i, \mu_j) &= 
    \sum_{v \in \bm{V}} w_v \left| \sum_{x \in \bm{V}_{\text{leaf}}} P_{\text{sub}}(x | v) \left( \mu_i (x) - \mu_j(x) \right) \right|_{\alpha} \\
    &= \sum_{v \in \bm{V}} w_v \left| \mu_i(\Gamma(v)) - \mu_j(\Gamma(v)) \right|_{\alpha} \\
    &\xrightarrow[\alpha \to \infty]{} W_{d_{\mathcal{T}}} (\mu_i, \mu_j)
\end{align*}
\end{proof}

\subsection{Additional Theoretical Analyses}
In the formulation of the soft tree-Wasserstein distance,
all nodes are contained in the subtree rooted at the root $v_1$.
Furthermore, every node is contained in the subtree rooted at itself.
\begin{theorem}
For all $u \in \bm{V}$,
$P_{\text{sub}} (u | v_1) = 1$. 
\end{theorem}
\begin{proof}
We prove that the elements in the first row of $(\mathbf{I} - \mathbf{D}_{\text{par}})^{-1}$ are all one.
Because $\mathbf{D}_{\text{par}}$ satisfies the conditions of Theorem \ref{th:condition_of_adjacency_matrix},
we have that
\begin{align*}
 \mathbf{1}_{N}^\top \mathbf{D}_{\text{par}} &= \left(0, 1, \ldots, 1\right), \\
 \mathbf{1}_{N}^\top (\mathbf{I} - \mathbf{D}_{\text{par}}) &= \left(1, 0, \ldots, 0\right).
\end{align*}
Since there exists the inverse matrix $(\mathbf{I} - \mathbf{D}_{\text{par}})^{-1}$,
we multiply this inverse matrix with the above equation, yielding 
\begin{align*}
 \mathbf{1}_{N}^\top  = \left(1, 0, \ldots, 0\right) (\mathbf{I} - \mathbf{D}_{\text{par}})^{-1}.
\end{align*}
Therefore, the statement is true.
\end{proof}

\begin{theorem}
For all $v \in \bm{V}$, $P_{\text{sub}}(v | v)=1$.
\end{theorem}
\begin{proof}
We prove that the diagonal elements of $(\mathbf{I} - \mathbf{D}_{\text{par}})^{-1}$ are all one.
Because $\mathbf{I} - \mathbf{D}_{\text{par}}$ is an upper triangular matrix,
$(\mathbf{I} - \mathbf{D}_{\text{par}})^{-1}$ is an upper triangular matrix.
Because $\mathbf{I} - \mathbf{D}_{\text{par}}$ is an upper triangular matrix and all diagonal elements are one,
all its eigenvalues are one.
Then all eigenvalues of $(\mathbf{I} - \mathbf{D}_{\text{par}})^{-1}$ are one.
Therefore, the diagonal elements of $(\mathbf{I} - \mathbf{D}_{\text{par}})^{-1}$ are all one.
\end{proof}

\section{Additional Experimental Results}
\subsection{Additional Analyses of Batch Size}
Figure \ref{fig:batch_all_tree} presents the time consumption of the tree-based methods when varying the batch size on AMAZON.
Figure \ref{fig:batch_all_tree_exp_flowtree} illustrates the time consumption of the tree-based methods except for Flowtree. 
The results show that the time consumption of Quadtree increases linearly with the number of documents to be compared.
However, the time consumption for the STW distance to compute a single batch is almost the same even if the batch size increases.
As a result, if the batch size is sufficiently large, the STW distance is faster than that of Quadtree.
Note that we implement the TSW distance by using the same formulation as the STW distance, 
which can be computed on a GPU.
Figure \ref{fig:speed_batch_500}, \ref{fig:speed_batch_1000}, \ref{fig:speed_batch_2500}, and \ref{fig:speed_batch_5000} show the time consumption of all baseline methods and the STW distance when the batch size is varied from $500$, $1000$, $2500$, and $5000$.
We omit datasets that contain only the number of training data below the batch size.

\subsection{Additional Analyses of Depth Level}
For the TSW and STW distances,
we need to set the depth level of the tree as the hyperparameters. 
Figure \ref{fig:speed_depth} shows the time required to compare one document with $500$ documents
of the TSW and STW distances when varying the tree's depth level.
The results show that, even if the depth level of the tree increases, the time consumption is almost the same.

\subsection{Time Consumption on CPU}
In this section,
we show the time consumption of the STW distance on a CPU.
We implement the STW distance with sparse matrix multiplications in SciPy.
Table \ref{table:sparse} shows the time consumption of the STW distance with sparse matrix multiplications on a CPU.
Unfortunately, the results indicate that the STW distance with sparse matrix multiplications is slower than Quadtree.
However, Quadtree is written in C++ and highly tuned. 
That is, if we implement the STW distance in the same way as Quadtree,
the STW distance can be computed as fast as Quadtree on a CPU.

\begin{table*}[b!]
\vskip -0.1 in
\caption{Average time consumption to compare one document with $500$ documents on a CPU [ms].}
\label{table:sparse}
\begin{center}
\vskip 0.1in
\begin{small}
\begin{tabular}{lcccccc}
\toprule
 & TWITTER & AMAZON & CLASSIC & BBCSPORT & OHSUMED & REUTERS  \\
\midrule
Quadtree     & 0.13 & 0.25 & 0.16 & 0.77 & 0.39 & 0.17\\
STW (sparse) & 1.78 & 4.79 & 3.77 & 7.65 & 6.42 & 4.49\\
\bottomrule
\end{tabular}
\end{small}
\end{center}
\vskip -0.2in
\end{table*}

\subsection{Analyses of Soft Tree-Wasserstein Distance}
In the STW distance,
we learn the probability of the tree's parent-child relationships by using the label information of documents, 
then we select the most probable parent node for each node.
In this section, we show how this thresholding affects the accuracy.  
We refer to the STW distance with $\mathbf{D}_2$, which represents the probability of the parent-child relationship,
and smooth approximation of the L1 norm as the soft-smooth-STW distance
and the STW distance with smooth approximation of the L1 norm as the smooth-STW distance.
We show the results in Table \ref{table:soft_accuracy}.
By comparing the smooth-STW and soft-smooth-STW distances,
the results show that this thresholding reduces the accuracy by about $1\%$.

\subsection{Other Experimental Results}
We show the loss value in the training in Figure \ref{fig:loss}.

\begin{figure}[t!]
\begin{center}
\centerline{\includegraphics[width=1.0\hsize]{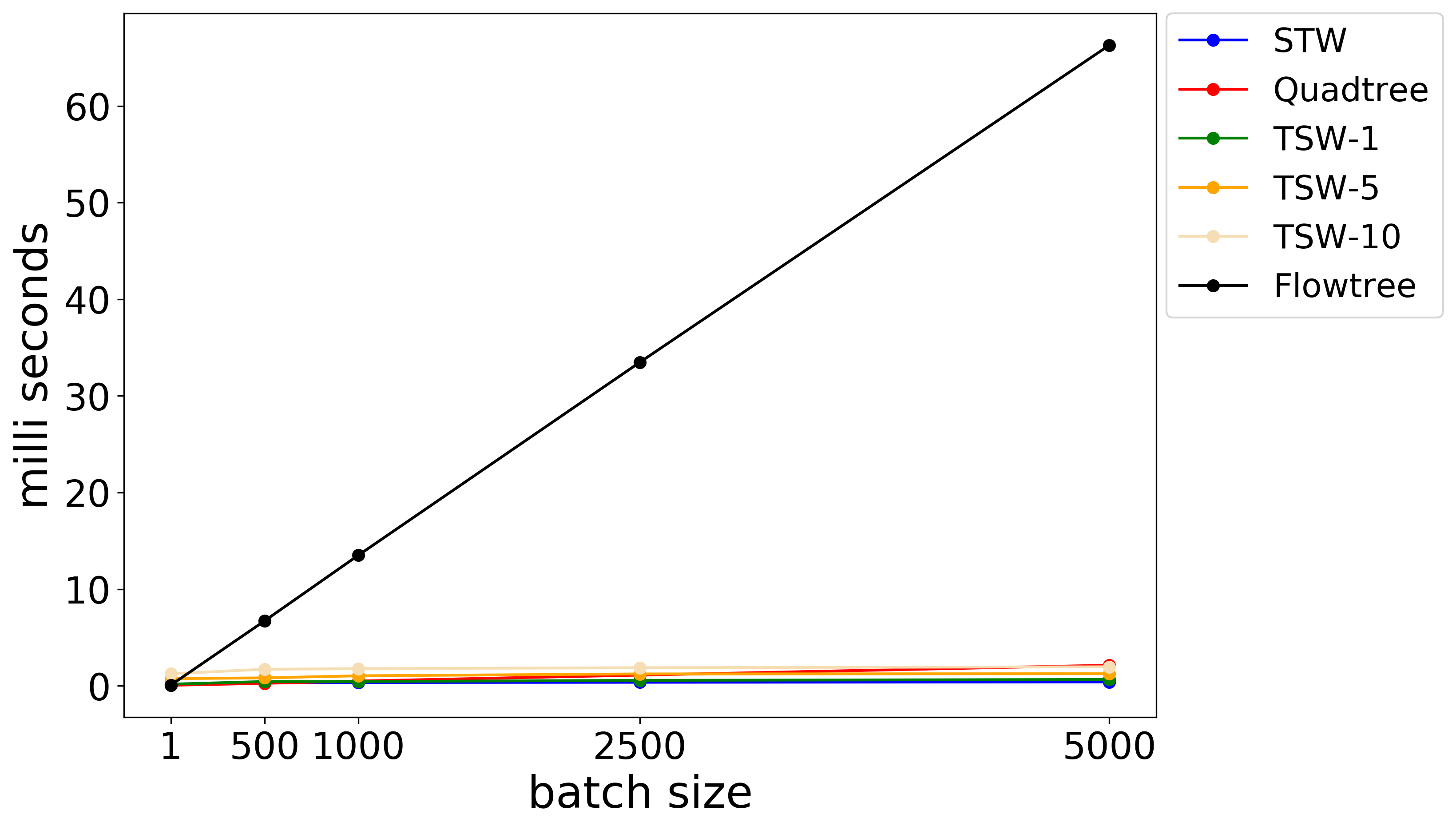}}
\caption{Average time consumption for all tree-based methods to compare one document
with the number of batch size documents on AMAZON.}
\label{fig:batch_all_tree}
\end{center}
\begin{center}
\centerline{\includegraphics[width=1.0\hsize]{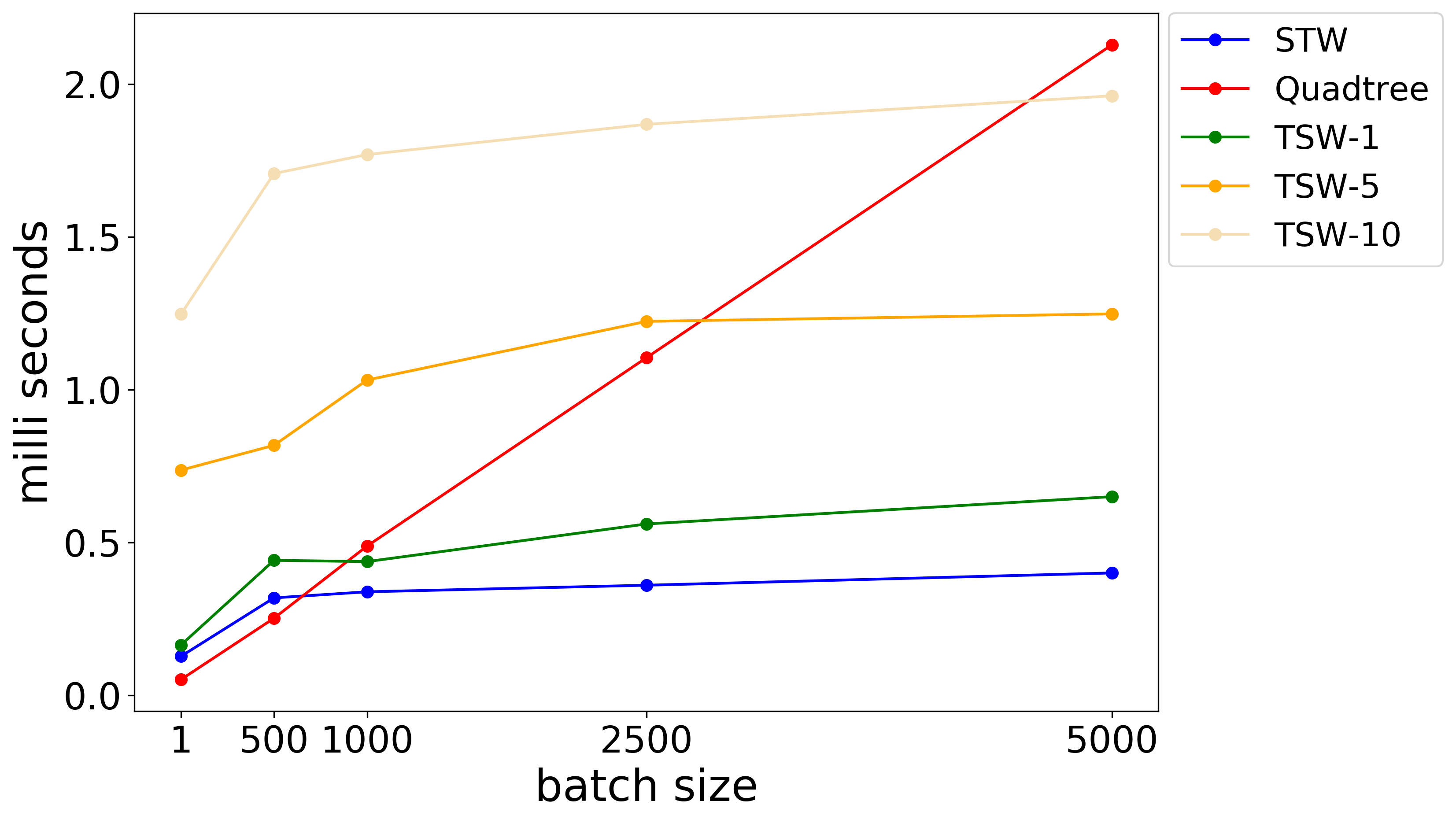}}
\caption{Average time consumption for comparing one document
with the number of batch size documents on AMAZON.}
\label{fig:batch_all_tree_exp_flowtree}
\end{center}
\vskip -0.2in
\end{figure}

\begin{figure*}[t!]
\vskip 0.2in
\begin{center}
\centerline{\includegraphics[width=\hsize]{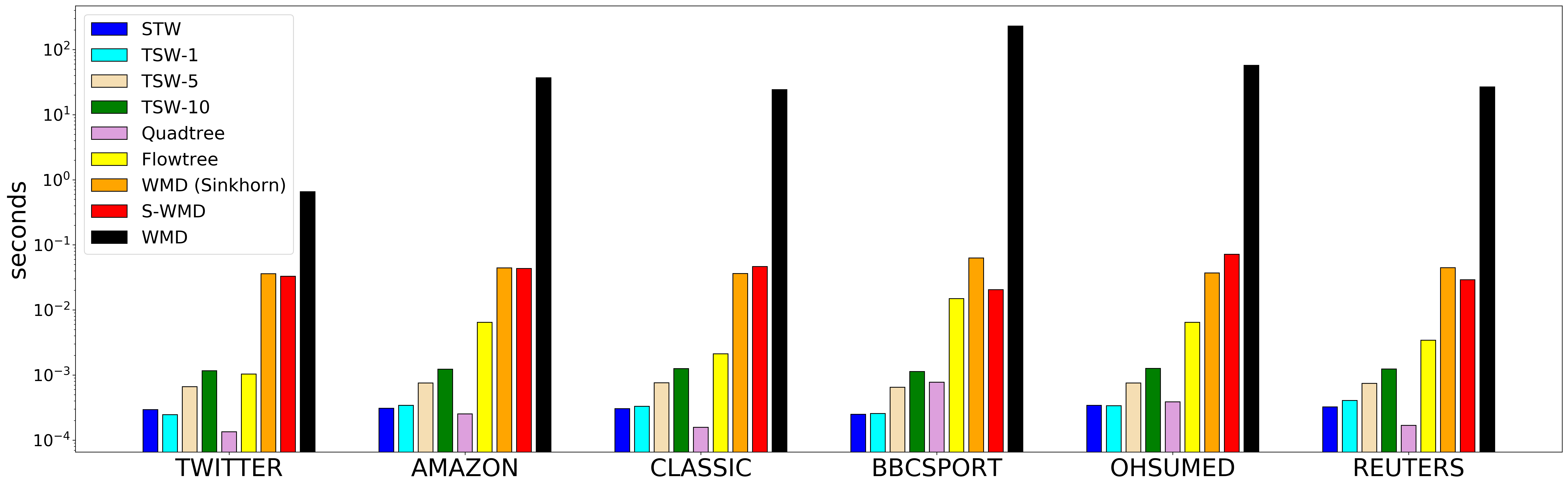}}
\caption{Average time consumption for comparing $500$ documents with one document.
For WMD (Sinkhorn), S-WMD, the STW distance, and the TSW distance, the batch size is set to $500$.}
\vskip -0.2in
\label{fig:speed_batch_500}
\end{center}
\end{figure*}

\begin{figure*}[t!]
\vskip 0.2in
\begin{center}
\centerline{\includegraphics[width=\hsize]{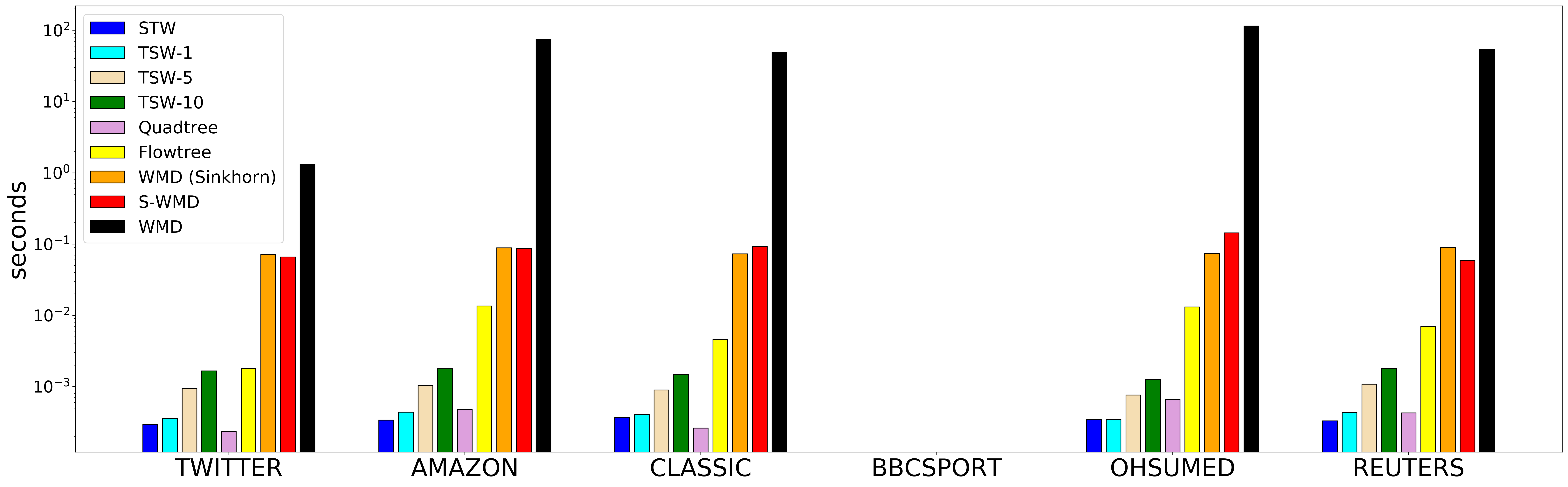}}
\caption{Average time consumption for comparing $1000$ documents with one document.
For the STW distance and the TSW distance, the batch size is set to $1000$. 
For WMD (Sinkhorn) and S-WMD, the batch size is set to $500$ due to the memory size limitations.}
\vskip -0.2in
\label{fig:speed_batch_1000}
\end{center}
\end{figure*}

\begin{figure*}[t!]
\begin{center}
\vskip 0.2in
\centerline{\includegraphics[width=\hsize]{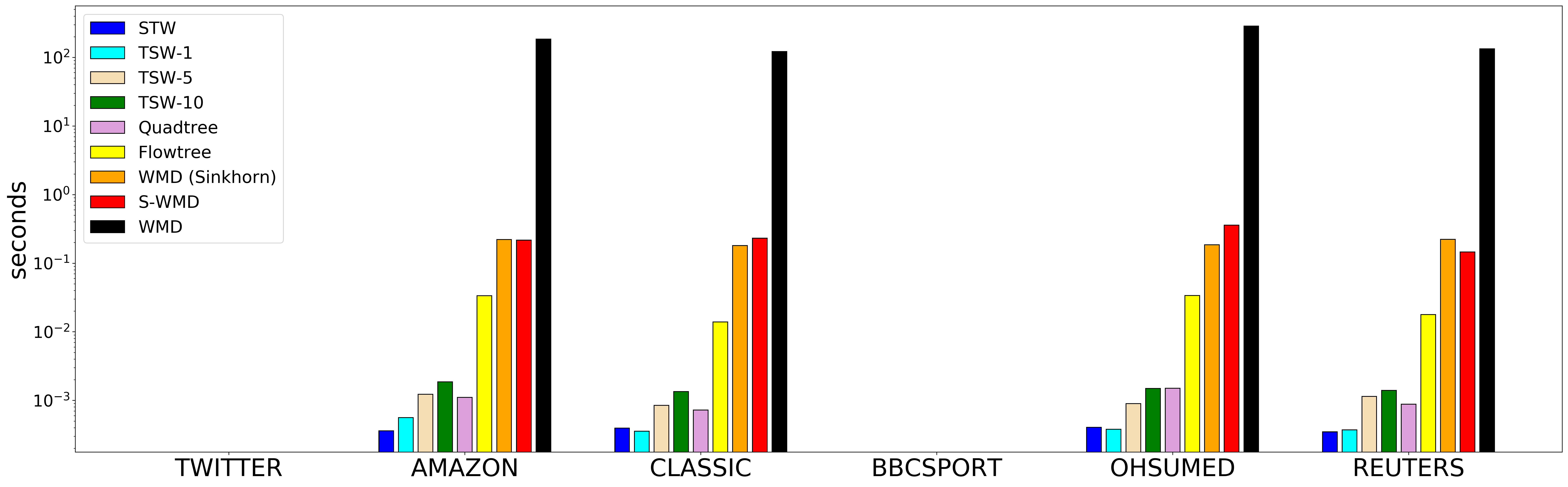}}
\caption{Average time consumption for comparing $2500$ documents with one document.
For the STW distance and the TSW distance, the batch size is set to $2500$. 
For WMD (Sinkhorn) and S-WMD, the batch size is set to $500$ due to the memory size limitations.}
\label{fig:speed_batch_2500}
\end{center}
\end{figure*}

\begin{figure*}[t!]
\begin{center}
\vskip 0.2in
\centerline{\includegraphics[width=\hsize]{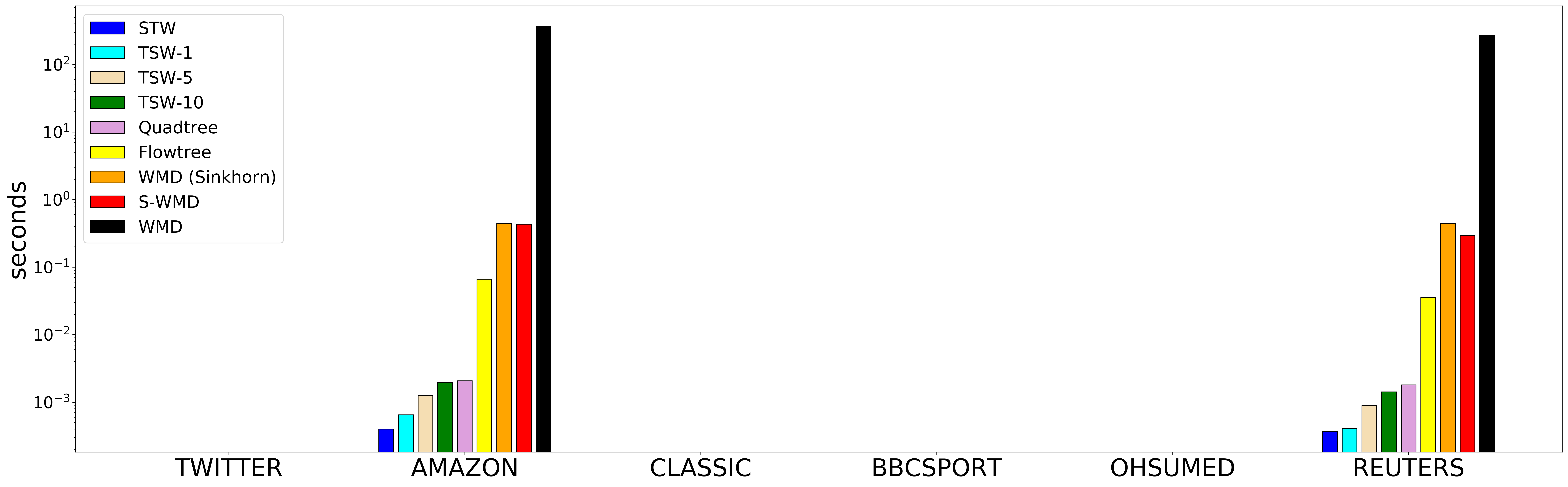}}
\caption{Average time consumption for comparing $5000$ documents with one document.
For the STW distance and the TSW distance, the batch size is set to $5000$. 
For WMD (Sinkhorn) and S-WMD, the batch size is set to $500$ due to the memory size limitations.}
\label{fig:speed_batch_5000}
\end{center}
\vskip -0.2in
\end{figure*}

\begin{figure*}[t!]
\vskip 0.2in
\begin{center}
\centerline{\includegraphics[width=\hsize]{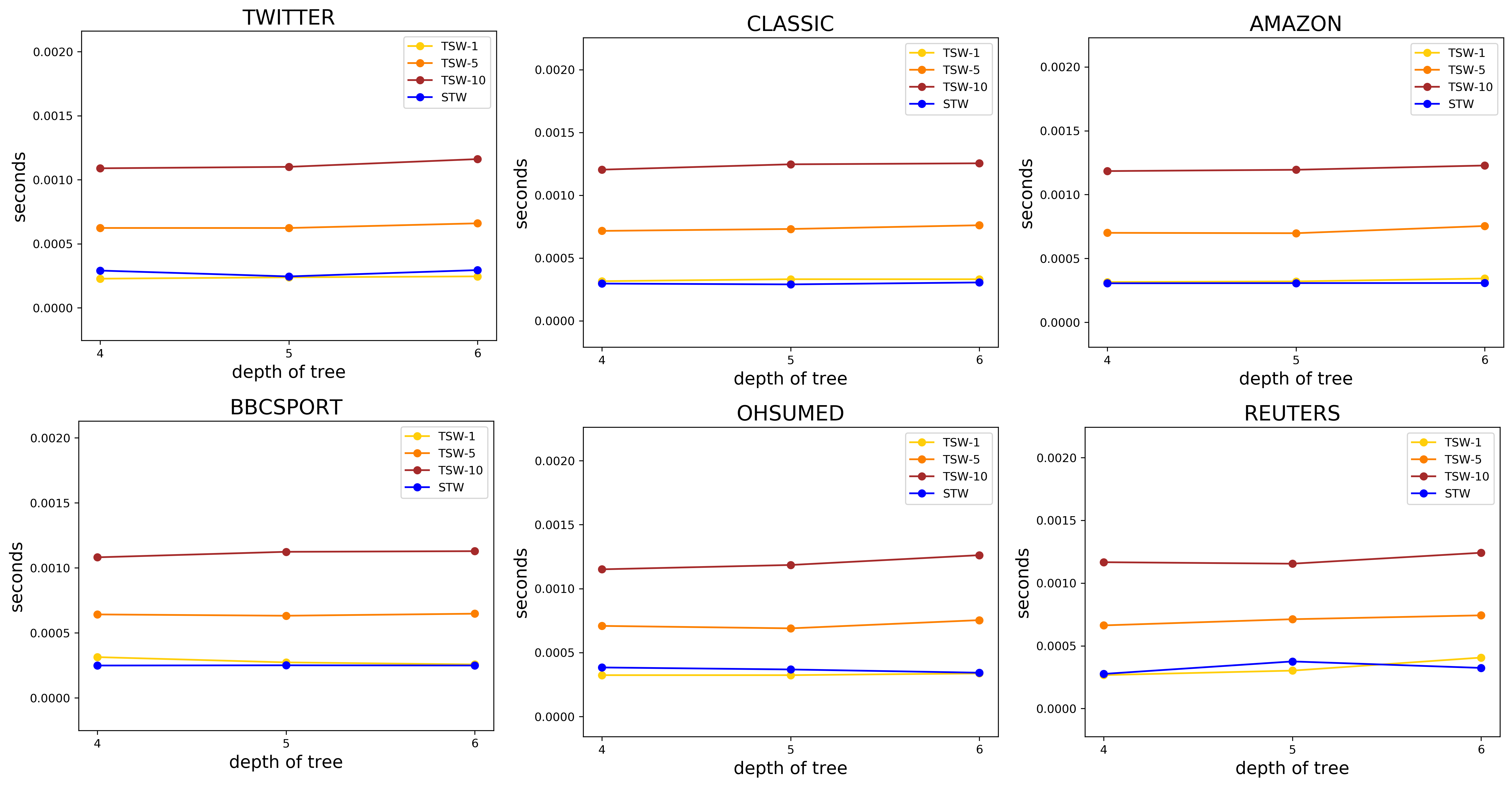}}
\caption{Average time consumption on all datasets for comparing one document with 500 documents when varying the depth level of the tree.}
\label{fig:speed_depth}
\end{center}
\vskip -0.2in
\end{figure*}

\begin{table*}[b]
\vskip -0.1 in
\caption{$k$NN test error rate. }
\label{table:soft_accuracy}
\vskip 0.1in
\begin{center}
\begin{small}
\begin{tabular}{lcccccc}
\toprule
 & TWITTER & AMAZON & CLASSIC & BBCSPORT & OHSUMED & REUTERS  \\
\midrule
soft-smooth-STW & 29.9 $\pm$ 1.3 & 8.4  $\pm$ 0.4 & 5.1 $\pm$ 0.2 & 4.5 $\pm$ 1.0 & 44.1 & 6.5 \\
smooth-STW      & 30.0 $\pm$ 0.8 & 10.6 $\pm$ 0.4 & 9.6 $\pm$ 0.9 & 4.5 $\pm$ 0.9 & 45.6 & 6.5 \\
STW             & 28.9 $\pm$ 0.7 & 10.1 $\pm$ 0.7 & 4.4 $\pm$ 0.7 & 3.4 $\pm$ 0.8 & 40.2 & 4.4 \\
\bottomrule
\end{tabular}
\end{small}
\end{center}
\vskip -0.2in
\end{table*}

\begin{figure*}[t!]
\vskip 0.2in
\begin{center}
\centerline{\includegraphics[width=\hsize]{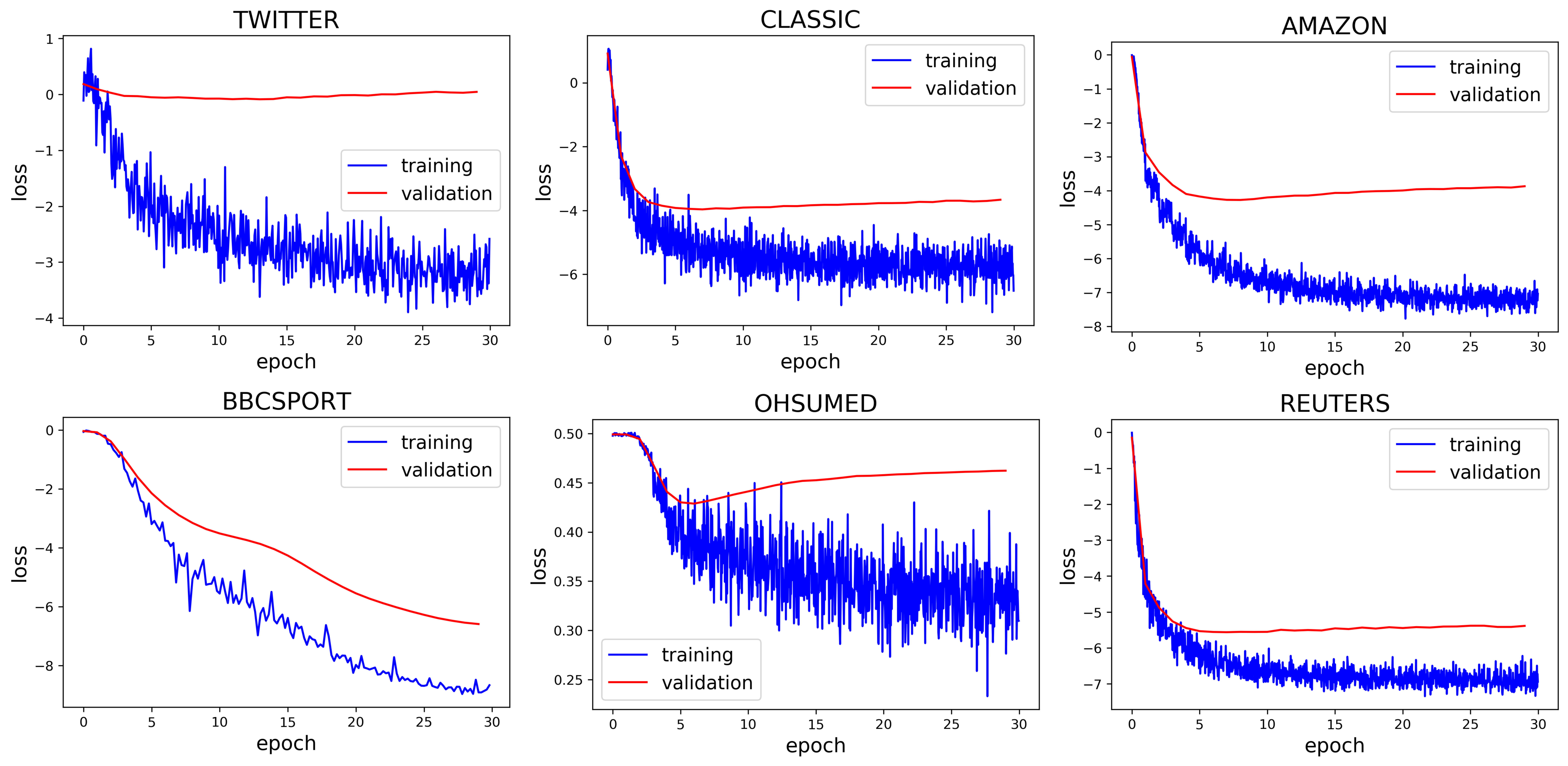}}
\caption{The loss value for all datasets.}
\label{fig:loss}
\end{center}
\vskip -0.2in
\end{figure*}


\end{document}